\documentclass{article}

\usepackage{microtype}
\usepackage{graphicx}
\usepackage{subfigure}
\usepackage{booktabs} 

\usepackage[hyphens]{url}
\usepackage{hyperref}
\hypersetup{breaklinks=true}

\usepackage[accepted]{icml2026}

\usepackage{amsmath}
\usepackage{amssymb}
\usepackage{mathtools}
\usepackage{amsthm}
\usepackage{csquotes}

\usepackage[capitalize,noabbrev]{cleveref}

\theoremstyle{plain}
\newtheorem{theorem}{Theorem}[section]

\theoremstyle{definition}
\newtheorem{definition}[theorem]{Definition}

\theoremstyle{remark}

\usepackage[textsize=tiny]{todonotes}
\usepackage{caption}

\usepackage{booktabs}
\usepackage{xcolor}    
\usepackage{pifont}    
\usepackage{makecell}

\usepackage{bm}
\usepackage{multirow}
\usepackage{colortbl}
\definecolor{color1}{rgb}{0.9764,0.8039,0.6784}

\newcommand{\cmark}{\textcolor{green}{\ding{51}}} 
\newcommand{\xmark}{\textcolor{red}{\ding{55}}}   

\icmltitlerunning{Expectation Consistency Loss for Confidence Calibration under Covariate Shift}

\begin{document}
	
\twocolumn[
\icmltitle{Expectation Consistency Loss: Rethink Confidence Calibration under Covariate Shift}



\icmlsetsymbol{equal}{*}

\begin{icmlauthorlist}
	\icmlauthor{Jinzong Dong}{sch}
	\icmlauthor{Zhaohui Jiang}{sch}
	\icmlauthor{Bo Yang}{sch}
\end{icmlauthorlist}

\icmlaffiliation{sch}{School of Automation, Central South University, Changsha, China}

\icmlcorrespondingauthor{Zhaohui Jiang}{jzh0903@csu.edu.cn}

\icmlkeywords{Machine Learning, ICML}
\vskip 0.3in
]



\printAffiliationsAndNotice{}  

\begin{abstract}
	Confidence calibration for classification models is vital in safety-critical decision-making scenarios and has received extensive attention. General confidence calibration methods assume training and test data are independent and identically distributed ($i.i.d.$), limiting their effectiveness under covariate shifts. Previous calibration methods under covariate shift struggle with class-wise or canonical calibrations and often rely on unstable importance weighting when density ratios are large or unbounded. Given the above limitations, this paper rethinks confidence calibration under covariate shifts. First, we derive a necessary and sufficient condition for confidence calibration under covariate shifts, named \textit{Expectation consistency condition}, which reveals covariate shifts do not necessarily lead to uncalibrated confidence and provides a weaker condition for confidence calibration than global covariate distribution alignment. Then, utilizing \textit{Expectation consistency condition}, this paper proposes an unsupervised domain adaptation loss to calibrate confidence of the target domain, named \textit{Expectation consistency loss} (ECL), which is compatible with canonical calibration, class-wise calibration, and top-label calibration. Third, we prove that computing ECL loss has the same sample complexity as Expected Calibration Error (ECE) and provide a theoretically grounded mini-batch trainable scheme for ECL loss. Finally, we validate the effectiveness of our method on both simulated and real-world covariate shift datasets.
\end{abstract}

\section{Introduction}
\label{Introduction}

Modern machine learning classification models, such as deep neural networks, are becoming increasingly accurate and widely applied in safety-critical fields \citep{LeCun2015,JIANG2023105849}. Nevertheless, decision-making systems in such fields need not only high accuracy but also the ability to recognize when they might be wrong \citep{NEURIPS2023_e271e30d}. For example, in automatic disease diagnosis, if a model has low confidence in its prediction, it should defer to a medical professional \citep{jiang2012calibrating}. Thus, classification models should provide accurate confidence estimates alongside their predictions to reflect the true likelihood of an event. Accurate confidence is more informative than mere class labels, e.g., stating \enquote{a patient has a 70\% probability of having cancer} gives doctors more actionable information than just labeling the condition as \enquote{cancer}. Moreover, accurate confidence facilitates classification models to better integrate with other probabilistic models, e.g., helping active learning to select more representative samples \citep{HAN2024110385} and improving the generalization performance of knowledge distillation \citep{li2023distilling}. Therefore, pursuing more accurate confidence in classification models is of great importance \citep{Gawlikowski2023}.

In recent years, confidence calibration has emerged as one of the most effective methods for producing more reliable confidence estimates and has attracted considerable attention \citep{guo17a,pmlr-v119-zhang20k,NEURIPS2019_8ca01ea9,dong2024combining}. However, general confidence calibration methods typically assume that the target domain (or test set) and the source domain (or calibration set) are independent and identically distributed ($i.i.d.$). When this assumption is violated due to distribution shifts, calibration performance often deteriorates significantly \citep{10356834}. Covariate shift, a common type of data distribution shift, often occurs in real-world tasks like medical diagnosis across different populations or image recognition under varying lighting conditions, where the input data distribution of models changes while the underlying relationship between inputs and outputs remains consistent \citep{kimura2024a}. Under covariate shift, models calibrated on the source domain frequently fail to generalize to the target domain, resulting in unreliable confidence estimates \citep{10.5555/1577069.1755858}. This highlights the importance of developing confidence calibration methods that remain robust under covariate shift \citep{pmlr-v235-hu24i}.

Currently, the mainstream confidence calibration methods under covariate shift are based on importance weighting \citep{DBLP:journals/corr/abs-2006-16405, pmlr-v108-park20b,NEURIPS2020_df12ecd0,10.24963/ijcai.2023/162}, which adjusts the objective function by assigning weights based on the importance of instances from the source domain, thereby guiding the model to generalize to the target domain unbiasedly \citep{kimura2024a}. However, it is well known that importance weighting has been criticized for its instability when the density ratio is large or unbounded \citep{NIPS2010_59c33016}. \citet{pmlr-v235-hu24i} use mixup to synthesize pseudo-target data and generalize the calibration performance from the pseudo-target data to the target domain. However, the efficacy of this method hinges primarily on the degree of similarity between the pseudo-target data and the target domain data. Furthermore, existing methods primarily address the simplest prediction-based calibration (i.e., top-label calibration). To our knowledge, there remains a notable absence of class-wise and canonical calibration methods designed to handle covariate shift.

Importance weighting in confidence calibration aims to globally align covariate distributions, inspired by accuracy improvement under covariate shift. However, confidence calibration differs fundamentally from accuracy improvement: it requires not learning new knowledge, but precisely conveying uncertainty. This raises a natural but often neglected question: \textbf{Is global covariate distribution alignment necessary?} To answer this, we first derive a necessary and sufficient condition for confidence calibration under covariate shifts, termed the \textit{Expectation consistency condition}. This condition reveals that covariate shifts do not necessarily cause miscalibration and provides a weaker requirement than global distribution alignment. Based on this condition, we propose an unsupervised domain adaptation loss, \textit{Expectation consistency loss} (ECL), with three variants for canonical, class-wise, and top-label calibration. We prove that ECL has sample complexity ${\cal O}(B/{\varepsilon^2})$, comparable to histogram binning, where $B$ denotes the number of confidence bins. To enable unbiased gradient backpropagation on mini-batch data, we also provide a theoretically sound mini-batch training scheme for ECL. Finally, we validate the method on simulated and real-world covariate shift datasets.

\begin{table*}[t]
	\centering
	\caption{Comparison of ECL and related calibration methods.}
	\setlength\tabcolsep{7.0pt}
	\renewcommand{\arraystretch}{1}
	\begin{tabular}{lccccc}
		\toprule
		\textbf{Method} & \makecell{\textbf{Covariate} \\ \textbf{Shift}} & \makecell{\textbf{Class-wise}\\ \textbf{Calibration}}&\makecell{\textbf{Canonical}\\ \textbf{Calibration}}& \makecell{\textbf{Density} \\ \textbf{Ratio} \\ \textbf{Unbounded}} & \makecell{\textbf{Mini-batch} \\ \textbf{Trainable}} \\
		\midrule
		$SB$-$ECE$ \citep{NEURIPS2021_f8905bd3}&\xmark&\xmark& \xmark&\cmark&\xmark\\
		$DECE$ \citep{bohdal2023metacalibration}&\xmark&\xmark& \xmark&\cmark&\xmark\\
		$ECE^{KDE}$ \citep{NEURIPS2022_33d6e648}&\xmark&\cmark& \cmark&\cmark&\cmark\\
		\midrule
		$Weighted$ $TS$ \citep{DBLP:journals/corr/abs-2006-16405}&\cmark&\xmark& \xmark&\xmark&\xmark\\
		$FL+IW+Temp$ \citep{pmlr-v108-park20b}&\cmark&\xmark& \xmark&\xmark&\xmark\\
		$TransCal$ \citep{NEURIPS2020_df12ecd0}&\cmark&\xmark& \xmark&\xmark&\xmark\\
		$DRL$ \citep{10.24963/ijcai.2023/162}&\cmark&\xmark& \xmark&\xmark&\xmark\\
		$PseudoCal$ \citep{pmlr-v235-hu24i}&\cmark&\xmark& \xmark&\cmark&\xmark\\
		\midrule
		$ECL$ (Ours)&\cmark&\cmark&\cmark&\cmark&\cmark\\
		\bottomrule
	\end{tabular}
	\label{Motivational_comparison}
\end{table*}
\section{Background and Related Work}
Consider a $K$-class classification problem where $X \in \mathcal{X}$ denotes the input feature and ${Y} = (Y_{1}, \cdots, Y_{K}) \in \mathcal{Y}$ denotes the $K$-class one-hot encoded label variable, with $\mathcal{X} \subset \mathbb{R}^{d}$ and $\mathcal{Y} = \{e_{k}\}_{k=1}^{K}$, where $e_{k}$ is a unit vector whose $k$-th component is 1. Let $f: \mathcal{X} \to \mathcal{S} \subset \Delta_{K-1}$ be a probabilistic classifier, where $\Delta_{K-1}$ represents a ($K-1$)-dimensional simplex. The predicted confidence score vector is given by $S = f(X)=(f_{1}(X),\cdots,f_{K}(X))= (S_{1}, {\ldots}, S_{K}) \in \mathcal{S}$. 
In general, the true class scalar is $Y^{*}={{\mathop{\rm argmax}\nolimits} _k}{\{ {Y_k}\} _{1 \le k \le K}}$, the predicted class is defined as $\hat Y = {{\mathop{\rm argmax}\nolimits} _k}{\{ {S_k}\} _{1 \le k \le K}}$, and the confidence score of the predicted class is $\hat S = \max {\{ {S_k}\} _{1 \le k \le K}}$.

In covariate shift, let $P_{s}(\cdot)$ and $P_{t}(\cdot)$ denote the probability density (for continuous variables, e.g., $X$, $X|S$, $S|X$, and $X|Y$) or probability measure (for discrete variables, e.g., $Y$, $Y|S$ and $Y|X$) on the source domain and target domain, respectively. $P$ denotes either $P_{s}$ or $P_{t}$ in cases where distinguishing between the source and target domains is not required. Let $D_{s}$ and $D_{t}$ represent the source domain and target domain data, respectively. 
\subsection{Confidence Calibration}
Confidence calibration aims to match the predicted confidence vector with the true posterior probability of event occurrence. Formally, we state:
\begin{definition}
	\textnormal{\textbf{(Perfect Calibration)}} A classifier is perfectly calibrated if the following equation holds:
	\begin{equation}
		P(Y_{k} = 1|S = s) = s_{k}, \forall 1 \le k \le K,
		\label{top_calibrated_eq}
	\end{equation}
	where $s=(s_{1},\cdots,s_{K})$ is the observed confidence score vector on $S$.
	\label{Multi_calibration}
\end{definition}

\textbf{Remark: }Definition \ref{Multi_calibration} considers the most stringent calibration paradigm, named canonical calibration \citep{Dong_2025}. Appendix \ref{calibration_paradigm} provides two other common calibration paradigms: top-label calibration \citep{guo17a} and class-wise calibration \citep{NEURIPS2019_8ca01ea9}. 

Existing general work primarily falls into two groups: train-time calibration \citep{Liu_2023_CVPR,NEURIPS2019_f1748d6b,9324926,Hebbalaguppe_2022_CVPR,Grathwohl2020Your,Yang_2021_ICCV} and post-hoc calibration \citep{guo17a, NEURIPS2019_8ca01ea9, pmlr-v119-zhang20k, NEURIPS2020_9bc99c59, gupta2021calibration, dong2024combining}. Train-time calibration typically carries out calibration during the classifier's training by adjusting the objective function, and post-hoc calibration learns a transformation (referred to as a calibration map) of the classifier's output on a calibration dataset in a post-hoc manner. However, these methods' effectiveness hinges on the $i.i.d.$ assumption between the target and source domains. When covariate shift occurs, this i.i.d. assumption is violated, making it difficult for the methods above to effectively calibrate confidence.

\subsection{Confidence Calibration under Covariate Shift}
In covariate shift, the target domain and the source domain have different feature distributions but the same conditional distributions. Formally, we state:
\begin{definition}
	\textnormal{\textbf{(Covariate Shift)}} Covariate shift occurs if the following two conditions are satisfied: ${P_{s}(X) \ne P_{t}(X)}$ and ${P_{s}(Y|X) = P_{t}(Y|X)}$.
	\label{Covariate_shift}
\end{definition}

Table \ref{Motivational_comparison} summarizes the characteristics of related calibration methods in five key dimensions, including whether they can handle covariate shifts, whether they support class-wise/canonical calibration, whether they can handle unbounded density ratios, and whether they are theoretically mini-batch trainable. As shown, existing methods often cover only a portion of the capabilities. In contrast, our ECL satisfies all dimensions simultaneously, demonstrating the method's comprehensiveness and versatility.

\section{Method}
\subsection{Expectation Consistency Condition}
Previous studies \citep{DBLP:journals/corr/abs-2006-16405,NEURIPS2020_df12ecd0,pmlr-v108-park20b,10.24963/ijcai.2023/162,pmlr-v235-hu24i} have empirically demonstrated that covariate shift can cause the confidence calibrated on the source domain to be uncalibrated on the target domain. However, empirical evidence alone cannot capture all possible scenarios. The theoretical underpinnings of these observations deserve to be explored to support this problem further and help solve it. To address this, this paper derives a necessary and sufficient condition for confidence calibration under covariate shift, as shown in Theorem \ref{Effect}.

\begin{theorem}
	\textnormal{\textbf{(Expectation Consistency Condition)}} $\forall 1 \le k \le K$, $P_{s}(Y_{k}=1|S)=P_{t}(Y_{k}=1|S)$ if and only if: ${\mathbb{E}_{X \sim {P_s}(X|S)}}[P(Y_{k}=1|X)] = {\mathbb{E}_{X \sim {P_t}(X|S)}}[P(Y_{k}=1|X)]$, where ${P(Y_{k}=1|X)} = {P_{s}(Y_{k}=1|X)} = {P_{t}(Y_{k}=1|X)}$. The proof is provided in Appendix \ref{proof_of_ECC}.
	\label{Effect}
\end{theorem}

\textbf{Remark on Theorem \ref{Effect}:} The source domain can usually be easily calibrated well using general calibration methods, at least much better than the target domain (see Appendix \ref{Source_and_Target}). Theorem \ref{Effect} tells us that as long as \textit{Expectation consistency condition} is met, the target domain can be calibrated as well as the source domain. Condition ${\mathbb{E}_{X \sim {P_s}(X|S)}}[P(Y_{k}=1|X)] = {\mathbb{E}_{X \sim {P_t}(X|S)}}[P(Y_{k}=1|X)]$ is strictly weaker than covariate distribution alignment (i.e., $P_{s}(X)=P_{t}(X)$), as it only requires equivalence in the expectations of the true posterior probability $P(Y_{k}=1|X)$ $w.r.t.$ the confidence score's level set distribution (i.e., $P_{s}(X|S)$ or $P_{t}(X|S)$), rather than matching the entire input distribution. For instance, even if $P_{s}(X)$ and $P_{t}(X)$ differ significantly, calibration may still hold if the model’s expected accuracy conditioned on $S$ aligns across domains. This insight moves the focus from aligning global covariate distributions to enforcing local consistency in critical statistics, enabling more efficient calibration strategies under covariate shift.

\textbf{Extension of Theorem \ref{Effect}:} Theorem \ref{Effect} can be naturally extended to top-label calibration and class-wise calibration (see Appendix \ref{Extension}). Intuitively, this only requires replacing the confidence score vector $S$ in Theorem \ref{Effect} with the predicted class confidence $\hat{S}$ or the confidence score vector's components $S_{k}$.

\textbf{An Example:} Fig. \ref{fig:side:a} shows an example of Theorem \ref{Effect}, where covariate shift occurs but calibration error remains unchanged. Take $S_{1}=0.75$ as an example for calculation: 
\begin{equation}
	\begin{split}
        & {P_s}\left(Y_{1} = 1|S=(0.75, 0.25)\right)\\
        & = \sum\limits_{X \in \{  - 1,1\} } {P(Y_{1} = 1|X){P_s}(X|S=(0.75, 0.25))} \\
        & = \sum\limits_{X \in \{  - 1,1\} } {0.5 \cdot {P_s}(X|S=(0.75, 0.25))}  = 0.5.
    \end{split}
\end{equation}
Similarly, it is easy to calculate that ${P_t}(Y_{1} = 1|S=(0.75, 0.25))=0.5={P_s}(Y_{1} = 1|S=(0.75, 0.25))$. The same holds if 0.75 is replaced with other values because ${\mathbb{E}_{X \sim {P_s}(X|S)}}[P(Y_{1} = 1|X)] = {\mathbb{E}_{X \sim {P_t}(X|S)}}[P(Y_{1} = 1|X)]$ holds for $\forall S_{1} \in [0,1]$. Moreover, such examples are infinite because they include but are not limited to all examples where $P(Y_{1}=1|X)$ or $S_{1}$ curves in Fig. \ref{fig:side:a} are symmetric $w.r.t.$ the y-axis.

\begin{figure}[t]
	\begin{center}
		\centerline{\includegraphics[width=\columnwidth]{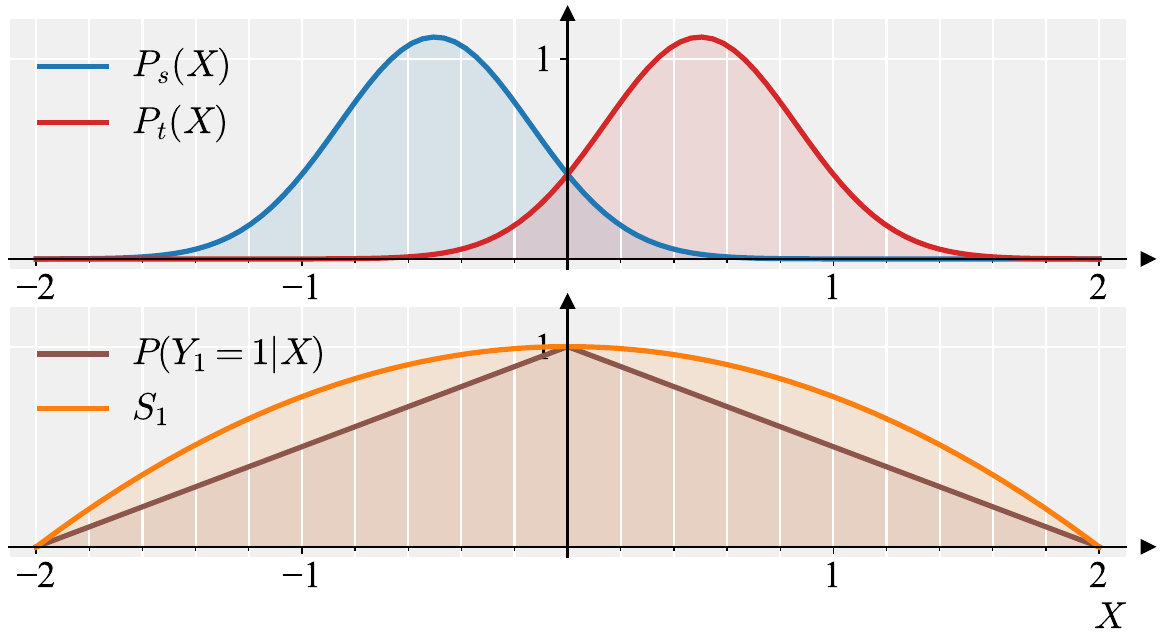}}
		\caption{A binary classification example where covariate shift occurs but calibration error remains unchanged, where $P(Y|X)=\left(P(Y_{1}|X),P(Y_{2}|X)\right)$ and $S=(S_{1}, S_{2})$. $P(Y_{2}|X)=1-P(Y_{1}|X)$ and $S_{2}=1-S_{1}$. ${P_s}(X) = {\left( {\sqrt {2\pi } } \right)^{ - 1}}{e^{ - 0.5{{(X + 0.5)}^2}}}$, ${P_t}(X) = {\left( {\sqrt {2\pi } } \right)^{ - 1}}{e^{ - 0.5{{(X - 0.5)}^2}}}$, $S_{1} = -0.25X^{2}+1$, and $P(Y_{1}=1|X)=-0.5|X|+1$.}
		\label{fig:side:a}
	\end{center}
	\vskip -0.2in
\end{figure}

\subsection{Expectation Consistency Loss}
\label{RCFM}
According to Theorem \ref{Effect}, \textit{Expectation consistency condition} ensures that the target domain can be calibrated as effectively as the source domain. Specifically, in canonical calibration, \textit{Expectation consistency condition} can be rewritten as follows:
\begin{equation} 	
\mathop \mathbb{E}\limits_{{P_t}(S)} \left\lVert {{\mathop  \mathbb{E}\limits_{{P_s}(X|S)}}P(Y|X) - {\mathop \mathbb{E}\limits_{{P_t}(X|S)}}P(Y|X)} \right\rVert = 0,
\end{equation}
where $S = (S_{1}, \cdots, S_{K})=(f_{1}(X),\cdots,f_{K}(X))=f(X)$, $P_{t}(S)$ represents the probability density of the predicted confidence score vector on the target domain. Therefore, \textit{Expectation consistency loss} can be naturally constructed as:
\begin{equation}
  L_{ecl} = \mathop \mathbb{E}\limits_{{P_t}(S)} \left\lVert {{\mathop \mathbb{E}\limits_{{P_s}(X|S)}}P(Y|X) - {\mathop \mathbb{E}\limits_{{P_t}(X|S)}}P(Y|X)} \right\rVert, 
	\label{con}
\end{equation}
To estimate $P(Y|X)$ in practice, we train an additional classification head on the original classifier's backbone, where the label is the one-hot encoded $Y$ and the input data is $X$. This classification head can be trained end-to-end with the original classifier (freeze the backbone when training this classification head). Optionally, this classification head can also be calibrated on the source domain.

\textbf{Extension of Expectation Consistency Loss: }Eq. \ref{con} is \textit{Expectation consistency loss} for canonical calibration. Similarly, \textit{Expectation consistency loss} for class-wise and top-label calibration can be obtained (see Appendix \ref{Loss_Extension}).

\subsection{Empirical Calculation and Differentiability}
$L_{ecl}$ can be empirically estimated using confidence binning and Monte Carlo sampling:
\begin{equation}
	\begin{dcases}
			{ {{\hat L_{ecl}}} = \sum\limits_{j = 1}^B {\frac{{\sharp b_j^{(t)}}}{{\sharp {D_t}}} \left\lVert {{{\hat {\mathbb{E}}}_{s,j}} - {{\hat {\mathbb{E}}}_{t,j}}} \right\rVert} ,}\\
			{{\hat{\mathbb{E}}_{s,j}} = \frac{1}{{\sharp D_s^{(j)}}}\sum\limits_{x \in D_s^{(j)}} {\hat P(Y=y|X = x)} ,}\\
			{{\hat{\mathbb{E}}_{t,j}} = \frac{1}{{\sharp D_t^{(j)}}}\sum\limits_{x \in D_t^{(j)}} {\hat P(Y=y|X = x)} ,}
	\end{dcases}
	\label{estimation}
\end{equation}
where $B$ represents the number of bins, $b_j^{(t)}$ represents the $j$-th bin in the target domain, $\sharp b_{j}^{(t)}$ represents sample size of $b_{j}^{(t)}$, $\sharp D_{t}$ represents sample size of $D_{t}$, ${D_s^{(j)}}$ represents the level set of $b_j^{(t)}$ in the source domain, ${D_t^{(j)}}$ represents the level set of $b_j^{(t)}$ in the target domain, and ${{{\hat P}}(Y = y|X = x)}$ represents the observation of ${{{P}}(Y|X)}$.

\textbf{Differentiability:} The confidence binning operation in Eq. \ref{estimation} is non-differentiable \citep{NEURIPS2021_f8905bd3, bohdal2023metacalibration, NEURIPS2022_33d6e648}, so it cannot be directly used for classifier training. Therefore, a differentiable version is proposed below. Specifically, we replace hard bin membership with a smooth anchor-based assignment over confidence bins. For canonical calibration, the $i$-th confidence vector $S^{(i)}\in\Delta_{K-1}$ is a point in simplex. We introduce $B$ anchor points $a_j\in\Delta_{K-1}$ and define for the soft assignment of the $i$-th confidence vector $S^{(i)}$:
\begin{equation}
\omega_{ij}=\frac{\exp(-\lVert S^{(i)}-a_j\rVert_2^{2}/\tau)}{\sum_{r=1}^{B}\exp(-\lVert S^{(i)}-a_r\rVert_2^{2}/\tau)} ,
\end{equation}
with temperature $\tau>0$. Denoting $p^{(i)}=P(Y|X_{i})$ as the output of the additional classification head (as described in Section \ref{RCFM}), we obtain for each bin $j$ and domain $d\in\{s,t\}$:
\begin{equation}
\hat{\mathbb{E}}_{d,j} = \frac{\sum_{i} \omega^{d}_{ij} \, p^{(i)}}{\sum_{i} \omega^{d}_{ij} + \varepsilon}, \qquad n^{d}_j = \sum_i \omega^{d}_{ij},
\end{equation}
with a small stabilizer $\varepsilon>0$, where $\omega^{d}_{ij}$ represents the soft assignment in domain $d$. Then, the differentiable ECL is:
\begin{equation}
\hat L_{ecl} = \sum_{j=1}^{B} w_j \; \left\lVert\hat{\mathbb{E}}_{s,j}-\hat{\mathbb{E}}_{t,j}\right\rVert, \qquad w_j = \frac{n^{t}_j}{\sum_{r=1}^{B} n^{t}_r}.
\label{Differentiable_ECL}
\end{equation}
\textbf{Extension of Differentiable ECL: }Eq. \ref{Differentiable_ECL} is differentiable \textit{Expectation consistency loss} for canonical calibration. Similarly, differentiable \textit{Expectation consistency loss} for top-label and class-wise calibration can be obtained (see Appendix \ref{appendix: Empirical Calculation and Differentiability}).

\subsection{Sample Complexity Analysis}
\begin{theorem}
\textnormal{\textbf{(Sample Complexity of ECL Estimation)}} Let $\varepsilon\in(0,1)$ and $\delta\in(0,1)$. Consider the empirical ECL in Eq. \ref{estimation} (or Eq. \ref{Differentiable_ECL}) with $B$ bins, bin weights $w_j$ (target-domain proportions or their soft analogs), and per-bin sample counts $n^{t}_{j}$ and $n^{s}_{j}$. There exist absolute constants $C>0$ such that, with probability at least $1-\delta$,
\begin{equation}
 \big|\hat L_{ecl}-L_{ecl} \big| \le C\sqrt{\log\Big(\tfrac{2BK}{\delta}\Big)\sum_{j=1}^{B} w_j\left( \frac{1}{n^{t}_{j}}+\frac{1}{n^{s}_{j}}\right)}.
\end{equation}
Its proof is provided in Appendix \ref{proof_sample_complexity}.
\label{thm:sample_complexity}
\end{theorem}

\textbf{Remark on Theorem \ref{thm:sample_complexity}:} Theorem \ref{thm:sample_complexity} implies ECL has a similar sample complexity as histogram binning for ECE, namely $\mathcal{O}(B/\varepsilon^2)$, and the weights $w_j$ explicitly cap the influence of sparse bins. This sample complexity is also similar to that of some point estimation methods (e.g., maximum likelihood estimation with $\mathcal{O}(1/\varepsilon^2)$) and is feasible for most real-world learning tasks.

\subsection{Mini-Batch Trainability}
Most modern deep learning methods are trained using mini-batches, where a small subset of data is processed at each step to compute the loss and update the model via gradient descent. This poses a challenge for confidence calibration loss, since small sample batches often fail to provide sufficiently accurate estimates of calibration error. Similar to the widely used cross-entropy loss, mini-batch trainability requires that the gradient computed on a mini-batch be an unbiased estimate of the gradient over the entire dataset, i.e., ${E_{D_s^{\rm{m}},D_t^{\rm{m}}}}\left[ {{\nabla _\theta }\hat L_{ecl}^{{\rm{m}}}} \right] = {\nabla _\theta }{{\hat L}_{ecl}}$, where $D_s^{\rm{m}}$ and $D_t^{\rm{m}}$ represent mini-batches from the source and target domains, respectively. Therefore, we propose an equivalent formulation of Eq. \ref{Differentiable_ECL} and prove its mini-batch trainability, as established in Theorem \ref{mini_batch_thm}.

\begin{theorem}
	\textnormal{\textbf{(ECL Mini-Batch Trainability)}} Eq. \ref{mini_batch_eq} is asymptotically equivalent to Eq. \ref{Differentiable_ECL}, and it satisfies ${E_{D_s^{\rm{m}},D_t^{\rm{m}}}}\left[ {{\nabla _\theta }\hat L_{ecl}^{{\rm{mini}}}} \right] = {\nabla _\theta }{{\hat L}_{ecl}}$, and its proof is provided in Appendix \ref{proof_mini_batch_thm}:
\begin{equation}
\begin{aligned}
\hat L_{ecl}(\theta ,u_j^s,u_j^t) &= \sum_{j = 1}^B w_j \|u_j^s - u_j^t\| \\
&\quad + \sum_{j = 1}^B \sum_{i \in D_s} \omega_{i,j}^s \, \|u_j^s - p^{(i)}(\theta)\|^2 \\
&\quad + \sum_{j = 1}^B \sum_{i \in D_t} \omega_{i,j}^t \, \|u_j^t - p^{(i)}(\theta)\|^2,
\end{aligned}
\label{mini_batch_eq}
\end{equation}
where $u_j^s$ and $u_j^t$ are learnable parameters used to approximate $\hat{\mathbb{E}}_{s,j}$ and $\hat{\mathbb{E}}_{t,j}$ during the training process, and $p^{(i)}(\theta)$ denotes $P(Y|X_i)$ estimated by an additional classification head trained on the original classifier's backbone.
\label{mini_batch_thm}
\end{theorem}
\textbf{Remark on Theorem \ref{mini_batch_thm}:} Because nonlinear operators such as norms do not commute with expectations, computing Eq. \ref{Differentiable_ECL} directly on a mini-batch introduces bias into the gradient, as demonstrated in the proof of Theorem \ref{mini_batch_thm}. By introducing auxiliary variables ($u_{j}^{s}$ and $u_{j}^{t}$) for learning the expectation over the full dataset, Eq. \ref{mini_batch_eq} perfectly avoids this problem. Algorithm \ref{alg:mini_batch_updates} provides the pseudocode for the actual calculation of Eq. \ref{mini_batch_eq}. Specifically, $u_{j}^{s}$ and $u_{j}^{t}$ in Eq. \ref{mini_batch_eq} can be solved using alternating proximal updates \citep{Bolte2014}, as detailed in Algorithm \ref{alg:mini_batch_updates}.

\textbf{Extension of ECL Mini-Batch Training: }Algorithm \ref{alg:mini_batch_updates} is ECL mini-batch training for canonical calibration. Similarly, ECL mini-batch training for top-label and class-wise calibration can be obtained (see Appendix \ref{Algorithom_Appendix}).

\begin{algorithm}[t]
	\caption{ECL Mini-Batch Training.}
	\begin{algorithmic}[1]
	\STATE \textbf{Input:}
		\STATE \quad bins $j=1\ldots B$, hyperparameters $\alpha_{\text{ema}}$, $N_{\text{prox}}$, $\lambda$
		\STATE \quad $u_j^s = \mathbf{0} \in \mathbb{R}^K, \forall j$; $u_j^t = \mathbf{0} \in \mathbb{R}^K, \forall j$
	\FOR{each iteration}
		\STATE Sample mini-batches $D_s^m,D_t^m$; 
		\STATE Compute weights $\omega_{ij}^s,\omega_{ij}^t$;
		\STATE $n_{s,j} \leftarrow \sum_{i\in D_s^m}\omega_{ij}^s$, $m_{s,j} \leftarrow \sum_{i\in D_s^m}\omega_{ij}^s p^{(i)}(\theta)$;
		\STATE $n_{t,j} \leftarrow \sum_{i\in D_t^m}\omega_{ij}^t$, $m_{t,j} \leftarrow \sum_{i\in D_t^m}\omega_{ij}^t p^{(i)}(\theta)$;
		\STATE $w_j \leftarrow n_{t,j} / \sum_{r=1}^{B} n_{t,r}$; $L_{\text{ecl}} \leftarrow 0$;
		\FOR{each bin $j$}
			\STATE $u_s,u_t\leftarrow$ cached $u_j^s,u_j^t$
			\FOR{$i=1$ to $N_{\text{prox}}$}
				\STATE $v_s\leftarrow (m_{s,j}/n_{s,j})-u_t$, $\tau_s=\dfrac{w_j}{2 n_{s,j}}$
				\STATE $u_s\leftarrow u_t + \mathrm{shrink}(v_s,\tau_s)$
				\STATE $v_t\leftarrow (m_{t,j}/n_{t,j})-u_s$, $\tau_t=\dfrac{w_j}{2 n_{t,j}}$
				\STATE $u_t\leftarrow u_s + \mathrm{shrink}(v_t,\tau_t)$
			\ENDFOR
			\STATE $\tilde u_j^s,\tilde u_j^t\leftarrow u_s.{\rm detach}(),\;u_t.{\rm detach}()$
			\STATE $u_j^s\leftarrow(1-\alpha_{\text{ema}})u_j^s+\alpha_{\text{ema}}\tilde u_j^s$
			\STATE $u_j^t\leftarrow(1-\alpha_{\text{ema}})u_j^t+\alpha_{\text{ema}}\tilde u_j^t$
			\STATE $L_{\text{ecl}} \mathrel{+}= \sum_{i\in D_s^m}\omega_{ij}^s\|\tilde u_j^s-p^{(i)}(\theta)\|^2$
			\STATE $L_{\text{ecl}} \mathrel{+}= \sum_{i\in D_t^m}\omega_{ij}^t\|\tilde u_j^t-p^{(i)}(\theta)\|^2$
		\ENDFOR
		\STATE Compute the cross-entropy loss $L_{\text{ce}}$
		\STATE Backpropagate $L_{\text{ce}} + \lambda L_{\text{ecl}}$ and update $\theta$.
	\ENDFOR
	\STATE \textbf{Return:} $\theta$
	\end{algorithmic}
	\label{alg:mini_batch_updates}
\end{algorithm}

\section{Results}
The effectiveness of the proposed method is verified from two perspectives: 1) Verify calibration
effectiveness on simulated covariate shift data; 2) Comparison with state-of-the-art calibration
methods on real-world covariate shift datasets.

\subsection{Calibration on Simulated Covariate Shift Data}
\begin{figure*}[t]
	\begin{center}
		\centerline{\includegraphics[width=1.0\textwidth]{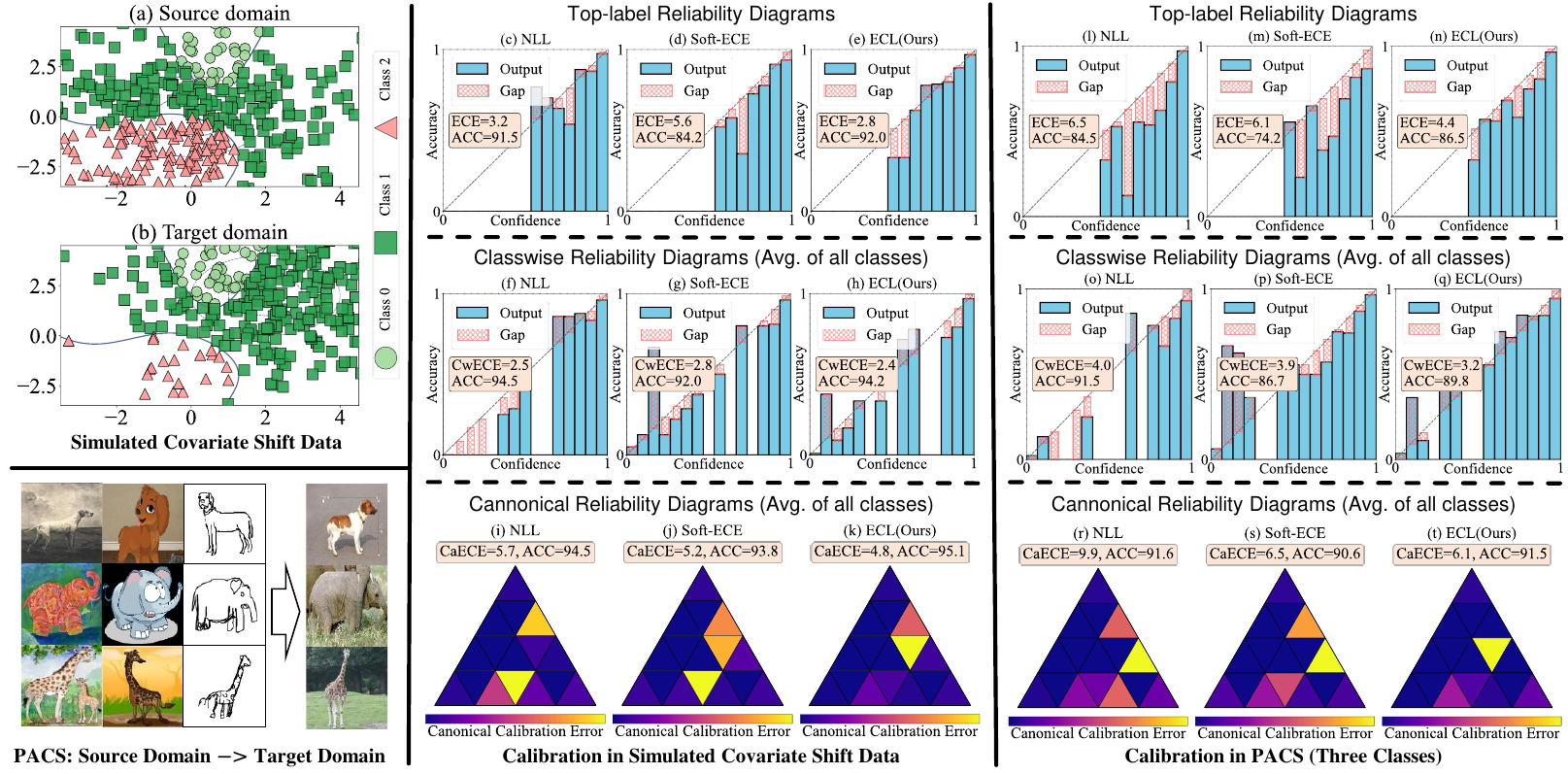}}
		\caption{Calibration effect display. Figures (c) to (k) show the calibration effect on the simulated covariate shift dataset (see Figures (a) and (b)), and Figures (l) to (t) show the calibration effect on the real-world covariate shift dataset PACS (three classes). NLL represents cross-entropy loss, Soft-ECE represents softened differentiable ECE loss, CwECE represents class-wise ECE, and CaECE represents canonical ECE. Results from the three types of reliability diagrams and calibration metrics demonstrate that our method preserves or improves classifier accuracy while substantially reducing calibration errors. Our code is available at \url{https://github.com/NeuroDong/ECL}.}
		\label{fig_before_abstract}
	\end{center}
	\vskip -0.2in
\end{figure*}

\textbf{Experimental Setup:} To observe covariate shift, we model source and target domain covariates as normal and uniform distributions (Figs. \ref{fig_before_abstract}(a-b) and Figs. \ref{fig:uniform}(a-b), respectively). For normal distributions, source domain has mean [0, 0] and covariance [[5, 0], [0, 5]], while target domain has mean [2, 2] and the same covariance. For uniform distributions, source domain is 2D uniform on $[-2.5, 2.5]^2$ and target domain on $[-1.5, 3.5]^2$. Since $P_{s}(Y|X)=P_{t}(Y|X)$, the labeling function is identical in both domains, shown by the blue segmentation curves in Figs. \ref{fig_before_abstract}(a-b) and Figs. \ref{fig:uniform}(a-b). We sample 400 points from each domain. The classifier is a three-layer backpropagation neural network trained with Adam optimizer (learning rate 0.001) for 100 epochs. Reliability diagrams use 15 bins \citep{guo17a,pmlr-v119-zhang20k}. The classification head estimating $P(Y|X)$ (or $P(Y^{*}=\hat Y|X)$ for top-label calibration) is calibrated on the source domain using Soft-ECE loss.

\textbf{Results:} Fig. \ref{fig_before_abstract} and Fig. \ref{fig:uniform} (in Appendix \ref{Simulated_appendix}) show the calibration results on the simulated covariate shift dataset. Fig. \ref{fig_before_abstract} shows the case where the covariate distribution is normally distributed, and Fig. \ref{fig:uniform} shows the case where the covariate distribution is uniformly distributed. The ECL’s results shown in the different reliability diagrams are from different ECL versions about different calibration paradigms. In the reliability diagrams, the outputs of the top-label and class-wise reliability diagrams after ECL calibration are closer to the diagonal, indicating improved calibration performance. In canonical calibration reliability diagrams, high calibration errors usually occur near the midpoint of a side of the large triangle, corresponding to situations where the confidence scores of each component of the predicted vector are not very high. Overall, the number of highlighted small triangle bins after ECL calibration in canonical reliability diagrams will decrease (see Fig. \ref{fig_before_abstract}) or the color will become dark blue (see Fig. \ref{fig:uniform}). From evaluating metrics under the two covariate distribution shifts, ECL can stably reduce calibration error in all three calibration paradigms and improve accuracy in most cases.

\subsection{Calibration on Real-World Covariate Shift Datasets}
\begin{table*}[t]
    \centering
    \footnotesize
    \caption{ECE (\%) for top-label calibration on digit recognition datasets. The reported results represent the mean and standard deviation derived from ten runs.}
    \setlength\tabcolsep{1.5pt}
    \renewcommand{\arraystretch}{1.}
    \begin{tabular}{clccccccccc|cc}
        \toprule
        \multicolumn{2}{c}{\multirow{2}{*}{\textbf{Datasets}}}&\multicolumn{9}{c|}{\textbf{ECE} $\bm{\downarrow}$}& \\
        & &Uncal&Soft-ECE&DECE&KDE&TS&TransCal&DRL&PseudoCal&\cellcolor{color1}ECL (Ours)&\multirow{-2}{*}{\textbf{Oracle} $\bm{\downarrow}$}&\multirow{-2}{*}{\textbf{$\Delta$ACC(\%)}}\\
        \midrule
        \multirow{12}{*}{\rotatebox{90}{\textbf{Digit}}}&\textbf{$\to$ \textit{MNIST}}& & & & & & & & & & & \\
        &LeNet-5&27.3$_{\pm 2.63}$&27.8$_{\pm 2.15}$&26.5$_{\pm 1.88}$&27.9$_{\pm 2.01}$&27.7$_{\pm 1.34}$&26.9$_{\pm 1.16}$&22.3$_{\pm 2.04}$&9.08$_{\pm 0.71}$&\cellcolor{color1}8.52$_{\pm 0.78}$&0.30$_{\pm 0.01}$&-0.92$_{\pm 0.35}$\\
        &ResNet20&16.2$_{\pm 1.51}$&16.5$_{\pm 1.22}$&15.8$_{\pm 1.45}$&16.1$_{\pm 1.10}$&15.3$_{\pm 1.04}$&13.1$_{\pm 0.99}$&10.2$_{\pm 0.72}$&8.22$_{\pm 0.53}$&\cellcolor{color1}7.88$_{\pm 0.45}$&1.54$_{\pm 0.04}$&+1.25$_{\pm 0.42}$\\
        &DenseNet40&23.4$_{\pm 1.79}$&23.6$_{\pm 1.55}$&22.1$_{\pm 1.62}$&22.9$_{\pm 1.48}$&21.6$_{\pm 1.71}$&19.8$_{\pm 0.96}$&14.8$_{\pm 0.95}$&9.72$_{\pm 0.68}$&\cellcolor{color1}9.15$_{\pm 0.61}$&1.40$_{\pm 0.03}$&+0.68$_{\pm 0.20}$\\
        \cmidrule{2-13}
        &\textbf{$\to$ \textit{USPS}}& & & & & & & & & & & \\
        &LeNet-5&22.9$_{\pm 1.50}$&23.1$_{\pm 1.28}$&22.4$_{\pm 1.40}$&22.8$_{\pm 1.35}$&22.7$_{\pm 1.13}$&21.8$_{\pm 1.32}$&15.5$_{\pm 1.16}$&8.92$_{\pm 0.45}$&\cellcolor{color1}8.12$_{\pm 0.42}$&1.54$_{\pm 0.02}$&-0.85$_{\pm 0.25}$\\
        &ResNet20&9.14$_{\pm 0.74}$&9.32$_{\pm 0.65}$&9.05$_{\pm 0.71}$&9.45$_{\pm 0.55}$&9.12$_{\pm 0.84}$&8.36$_{\pm 0.45}$&7.99$_{\pm 0.66}$&5.01$_{\pm 0.30}$&\cellcolor{color1}5.25$_{\pm 0.28}$&2.23$_{\pm 0.06}$&+1.42$_{\pm 0.37}$\\
        &DenseNet40&15.7$_{\pm 0.83}$&15.9$_{\pm 0.76}$&15.3$_{\pm 0.92}$&15.8$_{\pm 1.01}$&13.1$_{\pm 1.02}$&12.1$_{\pm 1.04}$&7.92$_{\pm 0.47}$&5.34$_{\pm 0.34}$&\cellcolor{color1}4.96$_{\pm 0.28}$&2.54$_{\pm 0.05}$&-0.76$_{\pm 0.18}$\\
        \cmidrule{2-13}
        &\textbf{$\to$ \textit{SVHN}}& & & & & & & & & & & \\
        &LeNet-5&61.9$_{\pm 6.16}$&62.2$_{\pm 5.50}$&60.8$_{\pm 5.22}$&62.5$_{\pm 5.80}$&61.3$_{\pm 5.89}$&63.7$_{\pm 4.94}$&23.7$_{\pm 1.93}$&52.4$_{\pm 4.55}$&\cellcolor{color1}21.5$_{\pm 1.51}$&1.03$_{\pm 0.02}$&+1.65$_{\pm 0.65}$\\
        &ResNet20&68.2$_{\pm 6.44}$&67.5$_{\pm 5.92}$&66.9$_{\pm 6.10}$&67.8$_{\pm 6.25}$&68.1$_{\pm 6.13}$&59.4$_{\pm 4.63}$&40.1$_{\pm 3.77}$&48.2$_{\pm 3.95}$&\cellcolor{color1}36.8$_{\pm 2.08}$&0.50$_{\pm 0.02}$&+2.12$_{\pm 0.88}$\\
        &DenseNet40&80.8$_{\pm 6.26}$&81.2$_{\pm 5.88}$&79.5$_{\pm 6.05}$&81.1$_{\pm 6.15}$&77.2$_{\pm 6.98}$&72.9$_{\pm 5.13}$&42.0$_{\pm 3.36}$&64.7$_{\pm 4.72}$&\cellcolor{color1}38.4$_{\pm 3.21}$&0.86$_{\pm 0.03}$&-1.15$_{\pm 0.45}$\\
        \bottomrule 
    \end{tabular}
    \label{comparison_on_Digit}
\end{table*}

\subsubsection{Experimental Setup}
\textbf{Datasets and Networks:} To reflect the effectiveness of calibration methods on the real-world dataset, three different types of covariate shift datasets are selected for experiments: 1) Digit recognition dataset includes three different domains (MNIST \citep{726791}, USPS \citep{291440}, and SVHN \citep{netzer2011reading}); 2) a domain adaptation dataset PACS contains four different domains (Photo, Art Painting, Cartoon, and Sketch) \citep{Li_2017_ICCV}; 3) a large-scale dataset ImageNet-Sketch with 1000 classes contains two domains (ImageNet and Sketch) \citep{NEURIPS2019_3eefceb8}. When constructing covariate shift datasets, one domain of the dataset is used as the target domain, and the other domains are merged into the source domain. The commonly used networks on these datasets are used in the experiments, i.e., LeNet \citep{726791}, ResNet \citep{He_2016_CVPR}, DenseNet \citep{Huang_2017_CVPR}, Wide-ResNet \citep{zagoruyko2016wide} and ViT \citep{dosovitskiy2021an}.

\textbf{Calibration Metrics:} To comprehensively evaluate the calibration performance in three calibration paradigms, we used the following calibration metrics to evaluate the calibration methods: 1) \textbf{ECE}: The classic expected calibration error \citep{guo17a} for top-label calibration; 2) \textbf{CwECE}: Class-wise expected calibration error \citep{NEURIPS2019_8ca01ea9} for class-wise calibration; 3) \textbf{ECE$^{\bm{KDE}}$}: a consistent and differentiable canonical calibration metric for canonical calibration. In addition, we report $\Delta$ACC as the accuracy change relative to the uncalibrated classifier under the same task/architecture, defined as $\Delta$ACC = $\mathrm{ACC}(\text{ECL}) - \mathrm{ACC}(\text{Uncal})$.

\textbf{Baselines:} For a comprehensive comparison, the following methods are compared: 1) \textbf{Uncal}: Training using only cross-entropy loss; 2) \textbf{Soft-ECE} \citep{NEURIPS2021_f8905bd3}: A softened differentiable ECE loss; 3) \textbf{DECE} \citep{bohdal2023metacalibration}: Another softened differentiable ECE loss; 4) \textbf{KDE} \citep{NEURIPS2022_33d6e648}: a differentiable canonical calibration loss; 5) \textbf{TS} \citep{guo17a}: Classic post-hoc calibration method with temperature scaling; 6) \textbf{TransCal} \citep{NEURIPS2020_df12ecd0}: a debiasing calibration method based on importance weighting; 7) \textbf{DRL} \citep{10.24963/ijcai.2023/162}: a calibration method based on distributionally robust learning; 8) \textbf{PseudoCal} \citep{pmlr-v235-hu24i}: a calibration method based on mixup data synthesis; 9) \textbf{Oracle}: Soft-ECE calibration using labels on the target domain.

\subsubsection{Results}

Table \ref{comparison_on_Digit} reports the calibration metric ECE for top-label calibration on the digit recognition benchmarks. Overall, ECL achieves the lowest (or near-lowest) ECE in most transfer tasks and network architectures, demonstrating strong calibration performance compared to state-of-the-art baselines. The advantage of ECL is particularly evident on the SVHN dataset, which involves larger distribution shifts; for instance, on LeNet-5, ECL reduces the ECE from 61.9\% (Uncalibrated) to 21.5\%, substantially improving upon most baselines (e.g., PseudoCal at 52.4\%). Furthermore, the $\Delta$ACC values suggest that ECL often improves calibration while largely preserving the discriminative power of the classifier.

Extended results covering broader benchmarks and calibration paradigms are detailed in Appendices \ref{Results_top_label}, \ref{Results_classwise}, and \ref{Results_canonical}. Appendix \ref{Results_top_label} provides additional top-label calibration results on the PACS and ImageNet-Sketch datasets. Appendices \ref{Results_classwise} and \ref{Results_canonical} present comprehensive evaluations for class-wise and canonical calibration, respectively, across all three dataset suites (Digit, PACS, and ImageNet-Sketch). Overall, these experiments show that ECL is highly competitive and frequently achieves the lowest errors in terms of ECE, CwECE, and ECE$^{KDE}$.

\subsection{Ablation}
\textbf{Mini-Batch Trainability:} 
We empirically verify the role of our proposed mini-batch training strategy by comparing it with a naive baseline, \textbf{Mini-Batch Non-Trainable ECL}, which directly computes Eq. \ref{Differentiable_ECL} on mini-batches. As shown in Table \ref{tab:trainable_vs_nontrainable} (see Appendix), our \textbf{Mini-Batch Trainable ECL} (Algorithm \ref{alg:mini_batch_updates}) is more stable and achieves better calibration in most settings, supporting the effectiveness of the auxiliary variable formulation (Theorem \ref{mini_batch_thm}).

\textbf{Loss Weight $\lambda$:} 
To balance the cross-entropy loss and ECL, we employ an adaptive weighting strategy: $\lambda = \beta^{\gamma}$, where $\beta  = \left( {\sum\nolimits_i {\mathcal{L}_{ce}^{(i)}} } \right)/\left( {\sum\nolimits_i {\mathcal{L}_{ecl}^{(i)}} } \right)$ acts as a balancing factor between the two loss magnitudes. The hyperparameter $\gamma$ controls the sensitivity of this regularization. Our ablation study in Table \ref{tab:ablation_gamma} (see Appendix) suggests that a linear scaling ($\gamma=1.0$) provides a strong trade-off between calibration improvement and accuracy preservation in our tested settings.

\section{Discussion}
\textbf{Why It Works:} 
The essence of ECL is to reorganize the confidence space rather than aligning covariate distributions. For each confidence level $S$, it ensures that source and target samples achieving this confidence share the same expected true posterior $P(Y|X)$, effectively grouping samples with similar true accuracy into the same confidence bins regardless of their input distributions. This level set alignment method directly addresses the essential need for confidence calibration under covariate shift, thereby achieving stable and effective calibration.

\textbf{Potential Impact, Limitations, and Future Work:} 
We rethink confidence calibration under covariate shifts by moving beyond traditional importance weighting. Our findings reveal that strict covariate distribution alignment is unnecessary; instead, a weaker condition—the \textit{Expectation Consistency Condition}—is sufficient for target domain calibration. This insight has the potential to inspire further research and enhance decision-making in safety-critical cross-population applications. However, our method assumes invariant posterior class probabilities ($P(Y|X)$), a common assumption among other methods in this field. Consequently, scenarios involving label shift, where the input-output relationship changes, fall outside the scope of this work. Future work will explore extending our framework to address calibration under both covariate and label shifts.

\section{Conclusion}

This paper rethinks confidence calibration under covariate shifts by moving beyond the traditional importance weighting paradigm. We derive a necessary and sufficient condition for confidence calibration under covariate shifts, termed the \textit{Expectation Consistency Condition}, which reveals that covariate shifts do not necessarily lead to uncalibrated confidence and provides a weaker condition than global covariate distribution alignment. Building upon this theoretical foundation, we propose the \textit{Expectation Consistency Loss} (ECL), an unsupervised domain adaptation loss that can be seamlessly applied to canonical, class-wise, and top-label calibration paradigms. Furthermore, we prove that ECL shares the same sample complexity as histogram binning for ECE estimation and provide a theoretically grounded mini-batch training scheme that enables unbiased gradient computation. Extensive experiments on both simulated and real-world covariate shift datasets demonstrate that ECL achieves competitive calibration errors across all three calibration paradigms while generally preserving classifier accuracy. Our work opens new avenues for confidence calibration research by shifting the focus from global distribution alignment to enforcing local consistency in critical statistics.

\section*{Acknowledgements}
This work was supported by the Science and Technology Innovation Program of Hunan Province (Grant Number: 2024RC1007) and the Central South University Post-Graduate Independent Exploration and Innovation Project (Grant Number: 2025ZZTS0616).

\section*{Impact Statement}
This paper presents work whose goal is to advance the field of Machine Learning. There are many potential societal consequences of our work, none which we feel must be specifically highlighted here.

\bibliography{example_paper}

@inproceedings{
	dosovitskiy2021an,
	title={An Image is Worth 16x16 Words: Transformers for Image Recognition at Scale},
	author={Alexey Dosovitskiy and Lucas Beyer and Alexander Kolesnikov and Dirk Weissenborn and Xiaohua Zhai and Thomas Unterthiner and Mostafa Dehghani and Matthias Minderer and Georg Heigold and Sylvain Gelly and Jakob Uszkoreit and Neil Houlsby},
	booktitle={International Conference on Learning Representations},
	year={2021},
	url={https://openreview.net/forum?id=YicbFdNTTy}
}

@inproceedings{zagoruyko2016wide,
	title={Wide Residual Networks},
	author={Zagoruyko, Sergey and Komodakis, Nikos},
	booktitle={Procedings of the British Machine Vision Conference 2016},
	pages={87--1},
	year={2016},
	organization={British Machine Vision Association}
}

@InProceedings{Huang_2017_CVPR,
	author = {Huang, Gao and Liu, Zhuang and van der Maaten, Laurens and Weinberger, Kilian Q.},
	title = {Densely Connected Convolutional Networks},
	booktitle = {Proceedings of the IEEE Conference on Computer Vision and Pattern Recognition (CVPR)},
	month = {July},
	year = {2017}
}

@ARTICLE{726791,
	author={Lecun, Y. and Bottou, L. and Bengio, Y. and Haffner, P.},
	journal={Proceedings of the IEEE}, 
	title={Gradient-based learning applied to document recognition}, 
	year={1998},
	volume={86},
	number={11},
	pages={2278-2324},
	doi={10.1109/5.726791}}

@InProceedings{He_2016_CVPR,
	author = {He, Kaiming and Zhang, Xiangyu and Ren, Shaoqing and Sun, Jian},
	title = {Deep Residual Learning for Image Recognition},
	booktitle = {Proceedings of the IEEE Conference on Computer Vision and Pattern Recognition (CVPR)},
	month = {June},
	year = {2016}
}

@ARTICLE{291440,
	author={Hull, J.J.},
	journal={IEEE Transactions on Pattern Analysis and Machine Intelligence}, 
	title={A database for handwritten text recognition research}, 
	year={1994},
	volume={16},
	number={5},
	pages={550-554},
	doi={10.1109/34.291440}}

@InProceedings{Li_2017_ICCV,
	author = {Li, Da and Yang, Yongxin and Song, Yi-Zhe and Hospedales, Timothy M.},
	title = {Deeper, Broader and Artier Domain Generalization},
	booktitle = {Proceedings of the IEEE International Conference on Computer Vision (ICCV)},
	month = {Oct},
	year = {2017}
}

@inproceedings{NEURIPS2019_3eefceb8,
	author = {Wang, Haohan and Ge, Songwei and Lipton, Zachary and Xing, Eric P},
	booktitle = {Advances in Neural Information Processing Systems},
	editor = {H. Wallach and H. Larochelle and A. Beygelzimer and F. d\textquotesingle Alch\'{e}-Buc and E. Fox and R. Garnett},
	pages = {},
	publisher = {Curran Associates, Inc.},
	title = {Learning Robust Global Representations by Penalizing Local Predictive Power},
	url = {https://proceedings.neurips.cc/paper_files/paper/2019/file/3eefceb8087e964f89c2d59e8a249915-Paper.pdf},
	volume = {32},
	year = {2019}
}

@inproceedings{NEURIPS2019_f1748d6b,
	author = {M\"{u}ller, Rafael and Kornblith, Simon and Hinton, Geoffrey E},
	booktitle = {Advances in Neural Information Processing Systems},
	editor = {H. Wallach and H. Larochelle and A. Beygelzimer and F. d\textquotesingle Alch\'{e}-Buc and E. Fox and R. Garnett},
	pages = {},
	publisher = {Curran Associates, Inc.},
	title = {When does label smoothing help?},
	url = {https://proceedings.neurips.cc/paper_files/paper/2019/file/f1748d6b0fd9d439f71450117eba2725-Paper.pdf},
	volume = {32},
	year = {2019}
}

@InProceedings{Liu_2023_CVPR,
	author    = {Liu, Bingyuan and Rony, J\'er\^ome and Galdran, Adrian and Dolz, Jose and Ben Ayed, Ismail},
	title     = {Class Adaptive Network Calibration},
	booktitle = {Proceedings of the IEEE/CVF Conference on Computer Vision and Pattern Recognition (CVPR)},
	month     = {June},
	year      = {2023},
	pages     = {16070-16079}
}

@ARTICLE{9324926,
	author={Fernando, K. Ruwani M. and Tsokos, Chris P.},
	journal={IEEE Transactions on Neural Networks and Learning Systems}, 
	title={Dynamically Weighted Balanced Loss: Class Imbalanced Learning and Confidence Calibration of Deep Neural Networks}, 
	year={2022},
	volume={33},
	number={7},
	pages={2940-2951},
	doi={10.1109/TNNLS.2020.3047335}}

@InProceedings{Hebbalaguppe_2022_CVPR,
	author    = {Hebbalaguppe, Ramya and Prakash, Jatin and Madan, Neelabh and Arora, Chetan},
	title     = {A Stitch in Time Saves Nine: A Train-Time Regularizing Loss for Improved Neural Network Calibration},
	booktitle = {Proceedings of the IEEE/CVF Conference on Computer Vision and Pattern Recognition (CVPR)},
	month     = {June},
	year      = {2022},
	pages     = {16081-16090}
}

@inproceedings{
	Grathwohl2020Your,
	title={Your classifier is secretly an energy based model and you should treat it like one},
	author={Will Grathwohl and Kuan-Chieh Wang and Joern-Henrik Jacobsen and David Duvenaud and Mohammad Norouzi and Kevin Swersky},
	booktitle={International Conference on Learning Representations},
	year={2020},
	url={https://openreview.net/forum?id=Hkxzx0NtDB}
}

@InProceedings{Yang_2021_ICCV,
	author    = {Yang, Xiulong and Ji, Shihao},
	title     = {JEM++: Improved Techniques for Training JEM},
	booktitle = {Proceedings of the IEEE/CVF International Conference on Computer Vision (ICCV)},
	month     = {October},
	year      = {2021},
	pages     = {6494-6503}
}

@inproceedings{NEURIPS2020_9bc99c59,
	author = {Rahimi, Amir and Shaban, Amirreza and Cheng, Ching-An and Hartley, Richard and Boots, Byron},
	booktitle = {Advances in Neural Information Processing Systems},
	editor = {H. Larochelle and M. Ranzato and R. Hadsell and M.F. Balcan and H. Lin},
	pages = {13456--13467},
	publisher = {Curran Associates, Inc.},
	title = {Intra Order-preserving Functions for Calibration of Multi-Class Neural Networks},
	url = {https://proceedings.neurips.cc/paper_files/paper/2020/file/9bc99c590be3511b8d53741684ef574c-Paper.pdf},
	volume = {33},
	year = {2020}
}

@inproceedings{
	gupta2021calibration,
	title={Calibration of Neural Networks using Splines},
	author={Kartik Gupta and Amir Rahimi and Thalaiyasingam Ajanthan and Thomas Mensink and Cristian Sminchisescu and Richard Hartley},
	booktitle={International Conference on Learning Representations},
	year={2021},
	url={https://openreview.net/forum?id=eQe8DEWNN2W}
}

@article{JIANG2023105849,
	title = {A novel intelligent monitoring method for the closing time of the taphole of blast furnace based on two-stage classification},
	journal = {Engineering Applications of Artificial Intelligence},
	volume = {120},
	pages = {105849},
	year = {2023},
	issn = {0952-1976},
	doi = {https://doi.org/10.1016/j.engappai.2023.105849},
	url = {https://www.sciencedirect.com/science/article/pii/S0952197623000337},
	author = {Zhaohui Jiang and Jinzong Dong and Dong Pan and Tianyu Wang and Weihua Gui}
}

@Article{LeCun2015,
	author={LeCun, Yann
	and Bengio, Yoshua
	and Hinton, Geoffrey},
	title={Deep learning},
	journal={Nature},
	year={2015},
	month={May},
	day={01},
	volume={521},
	number={7553},
	pages={436-444},
	issn={1476-4687},
	doi={10.1038/nature14539},
	url={https://doi.org/10.1038/nature14539}
}

@inproceedings{NEURIPS2023_e271e30d,
	author = {Munir, Muhammad Akhtar and Khan, Salman H and Khan, Muhammad Haris and Ali, Mohsen and Shahbaz Khan, Fahad},
	booktitle = {Advances in Neural Information Processing Systems},
	editor = {A. Oh and T. Naumann and A. Globerson and K. Saenko and M. Hardt and S. Levine},
	pages = {71619--71631},
	publisher = {Curran Associates, Inc.},
	title = {Cal-DETR: Calibrated Detection Transformer},
	url = {https://proceedings.neurips.cc/paper_files/paper/2023/file/e271e30de7a2e462ca1f85cefa816380-Paper-Conference.pdf},
	volume = {36},
	year = {2023}
}

@article{jiang2012calibrating,
	author = {Jiang, Xiaoqian and Osl, Melanie and Kim, Jihoon and Ohno-Machado, Lucila},
	title = {Calibrating predictive model estimates to support personalized medicine},
	journal = {Journal of the American Medical Informatics Association},
	volume = {19},
	number = {2},
	pages = {263-274},
	year = {2011},
	month = {10},
	issn = {1067-5027},
	doi = {10.1136/amiajnl-2011-000291},
	url = {https://doi.org/10.1136/amiajnl-2011-000291},
	eprint = {https://academic.oup.com/jamia/article-pdf/19/2/263/17374049/19-2-263.pdf},
}

@article{HAN2024110385,
	title = {BALQUE: Batch active learning by querying unstable examples with calibrated confidence},
	journal = {Pattern Recognition},
	volume = {151},
	pages = {110385},
	year = {2024},
	issn = {0031-3203},
	doi = {https://doi.org/10.1016/j.patcog.2024.110385},
	url = {https://www.sciencedirect.com/science/article/pii/S0031320324001365},
	author = {Yincheng Han and Dajiang Liu and Jiaxing Shang and Linjiang Zheng and Jiang Zhong and Weiwei Cao and Hong Sun and Wu Xie}
}

@inproceedings{li2023distilling,
	title = "Distilling Calibrated Knowledge for Stance Detection",
	author = "Li, Yingjie  and
	Caragea, Cornelia",
	editor = "Rogers, Anna  and
	Boyd-Graber, Jordan  and
	Okazaki, Naoaki",
	booktitle = "Findings of the Association for Computational Linguistics: ACL 2023",
	month = jul,
	year = "2023",
	address = "Toronto, Canada",
	publisher = "Association for Computational Linguistics",
	url = "https://aclanthology.org/2023.findings-acl.393/",
	doi = "10.18653/v1/2023.findings-acl.393",
	pages = "6316--6329"
}

@Article{Gawlikowski2023,
	author={Gawlikowski, Jakob
	and Tassi, Cedrique Rovile Njieutcheu
	and Ali, Mohsin
	and Lee, Jongseok
	and Humt, Matthias
	and Feng, Jianxiang
	and Kruspe, Anna
	and Triebel, Rudolph
	and Jung, Peter
	and Roscher, Ribana
	and Shahzad, Muhammad
	and Yang, Wen
	and Bamler, Richard
	and Zhu, Xiao Xiang},
	title={A survey of uncertainty in deep neural networks},
	journal={Artificial Intelligence Review},
	year={2023},
	month={Oct},
	day={01},
	volume={56},
	number={1},
	pages={1513-1589},
	issn={1573-7462},
	doi={10.1007/s10462-023-10562-9},
	url={https://doi.org/10.1007/s10462-023-10562-9}
}

@article{dong2024combining, 
	title={Combining Priors with Experience: Confidence Calibration Based on Binomial Process Modeling}, 
	volume={39}, 
	url={https://ojs.aaai.org/index.php/AAAI/article/view/33792}, 
	DOI={10.1609/aaai.v39i15.33792}, 
	abstractNote={Confidence calibration of classification models is a technique to estimate the true posterior probability of the predicted class, which is critical for ensuring reliable decision-making in practical applications. Existing confidence calibration methods mostly use statistical techniques to estimate the calibration curve from data or fit a user-defined calibration function, but often overlook fully mining and utilizing the prior distribution behind the calibration curve. However, a well-informed prior distribution can provide valuable insights beyond the empirical data under the limited data or low-density regions of confidence scores. To fill this gap, this paper proposes a new method that integrates the prior distribution behind the calibration curve with empirical data to estimate a continuous calibration curve, which is realized by modeling the sampling process of calibration data as a binomial process and maximizing the likelihood function of the binomial process. We prove that the calibration curve estimating method is Lipschitz continuous with respect to data distribution and requires smaller sample sizes than histogram binning. Also, a new calibration metric has been designed, leveraging the estimated calibration curve to estimate the true calibration error, and it has been proven to be a consistent calibration measure. Furthermore, realistic calibration datasets can be generated by the binomial process modeling from a preset true calibration curve and confidence score distribution, which can serve as a benchmark to measure and compare the discrepancy between existing calibration metrics and the true calibration error. The effectiveness of our calibration method and metric are verified in real-world and simulated data. We believe our exploration of integrating prior distributions with empirical data will guide the development of better-calibrated models, contributing to trustworthy AI.}, 
	number={15}, 
	journal={Proceedings of the AAAI Conference on Artificial Intelligence}, 
	author={Dong, Jinzong and Jiang, Zhaohui and Pan, Dong and Yu, Haoyang}, 
	year={2025}, 
	month={Apr.}, 
	pages={16317-16326} }

@ARTICLE{Dong_2025,
	author={Dong, Jinzong and Jiang, Zhaohui and Pan, Dong and Chen, Zhiwen and Guan, Qingyi and Zhang, Hongbin and Gui, Gui and Gui, Weihua},
	journal={IEEE Transactions on Neural Networks and Learning Systems}, 
	title={A Survey on Confidence Calibration of Deep Learning-Based Classification Models Under Class Imbalance Data}, 
	year={2025},
	volume={36},
	number={9},
	pages={15664-15684},
	doi={10.1109/TNNLS.2025.3565159}}

@InProceedings{guo17a,
	title = 	 {On Calibration of Modern Neural Networks},
	author =       {Chuan Guo and Geoff Pleiss and Yu Sun and Kilian Q. Weinberger},
	booktitle = 	 {Proceedings of the 34th International Conference on Machine Learning},
	pages = 	 {1321--1330},
	year = 	 {2017},
	editor = 	 {Precup, Doina and Teh, Yee Whye},
	volume = 	 {70},
	series = 	 {Proceedings of Machine Learning Research},
	month = 	 {06--11 Aug},
	publisher =    {PMLR},
	url = 	 {https://proceedings.mlr.press/v70/guo17a.html}
}

@InProceedings{pmlr-v119-zhang20k,
	title = 	 {Mix-n-Match : Ensemble and Compositional Methods for Uncertainty Calibration in Deep Learning},
	author =       {Zhang, Jize and Kailkhura, Bhavya and Han, T. Yong-Jin},
	booktitle = 	 {Proceedings of the 37th International Conference on Machine Learning},
	pages = 	 {11117--11128},
	year = 	 {2020},
	editor = 	 {III, Hal Daumé and Singh, Aarti},
	volume = 	 {119},
	series = 	 {Proceedings of Machine Learning Research},
	month = 	 {13--18 Jul},
	publisher =    {PMLR},
	url = 	 {https://proceedings.mlr.press/v119/zhang20k.html}
}

@inproceedings{NEURIPS2019_8ca01ea9,
	author = {Kull, Meelis and Perello Nieto, Miquel and K\"{a}ngsepp, Markus and Silva Filho, Telmo and Song, Hao and Flach, Peter},
	booktitle = {Advances in Neural Information Processing Systems},
	editor = {H. Wallach and H. Larochelle and A. Beygelzimer and F. d\textquotesingle Alch\'{e}-Buc and E. Fox and R. Garnett},
	pages = {},
	publisher = {Curran Associates, Inc.},
	title = {Beyond temperature scaling: Obtaining well-calibrated multi-class probabilities with Dirichlet calibration},
	url = {https://proceedings.neurips.cc/paper_files/paper/2019/file/8ca01ea920679a0fe3728441494041b9-Paper.pdf},
	volume = {32},
	year = {2019}
}

@ARTICLE{10356834,
	author={Zhu, Fei and Zhang, Xu-Yao and Cheng, Zhen and Liu, Cheng-Lin},
	journal={IEEE Transactions on Pattern Analysis and Machine Intelligence}, 
	title={Revisiting Confidence Estimation: Towards Reliable Failure Prediction}, 
	year={2024},
	volume={46},
	number={5},
	pages={3370-3387},
	doi={10.1109/TPAMI.2023.3342285}}

@article{10.5555/1577069.1755858,
	author = {Bickel, Steffen and Br\"{u}ckner, Michael and Scheffer, Tobias},
	title = {Discriminative Learning Under Covariate Shift},
	year = {2009},
	issue_date = {12/1/2009},
	publisher = {JMLR.org},
	volume = {10},
	number={9},
	issn = {1532-4435},
	journal = {Journal of Machine Learning Research},
	month = dec,
	pages = {2137–2155},
	numpages = {19}
}

@article{
	kimura2024a,
	title={A Short Survey on Importance Weighting for Machine Learning},
	author={Masanari Kimura and Hideitsu Hino},
	journal={Transactions on Machine Learning Research},
	issn={2835-8856},
	year={2024},
	url={https://openreview.net/forum?id=IhXM3g2gxg},
	note={Survey Certification}
}

@InProceedings{pmlr-v235-hu24i,
	title = 	 {Pseudo-Calibration: Improving Predictive Uncertainty Estimation in Unsupervised Domain Adaptation},
	author =       {Hu, Dapeng and Liang, Jian and Wang, Xinchao and Foo, Chuan-Sheng},
	booktitle = 	 {Proceedings of the 41st International Conference on Machine Learning},
	pages = 	 {19304--19326},
	year = 	 {2024},
	editor = 	 {Salakhutdinov, Ruslan and Kolter, Zico and Heller, Katherine and Weller, Adrian and Oliver, Nuria and Scarlett, Jonathan and Berkenkamp, Felix},
	volume = 	 {235},
	series = 	 {Proceedings of Machine Learning Research},
	month = 	 {21--27 Jul},
	publisher =    {PMLR},
	url = 	 {https://proceedings.mlr.press/v235/hu24i.html}
}

@InProceedings{pmlr-v108-park20b,
	title = 	 {Calibrated Prediction with Covariate Shift via Unsupervised Domain Adaptation},
	author =       {Park, Sangdon and Bastani, Osbert and Weimer, James and Lee, Insup},
	booktitle = 	 {Proceedings of the Twenty Third International Conference on Artificial Intelligence and Statistics},
	pages = 	 {3219--3229},
	year = 	 {2020},
	editor = 	 {Chiappa, Silvia and Calandra, Roberto},
	volume = 	 {108},
	series = 	 {Proceedings of Machine Learning Research},
	month = 	 {26--28 Aug},
	publisher =    {PMLR},
	url = 	 {https://proceedings.mlr.press/v108/park20b.html}
}

@article{DBLP:journals/corr/abs-2006-16405,
	publtype={informal},
	author={Anusri Pampari and Stefano Ermon},
	title={Unsupervised Calibration under Covariate Shift},
	year={2020},
	cdate={1577836800000},
	journal={CoRR},
	volume={abs/2006.16405},
	url={https://arxiv.org/abs/2006.16405}
}

@inproceedings{NEURIPS2020_df12ecd0,
	author = {Wang, Ximei and Long, Mingsheng and Wang, Jianmin and Jordan, Michael},
	booktitle = {Advances in Neural Information Processing Systems},
	editor = {H. Larochelle and M. Ranzato and R. Hadsell and M.F. Balcan and H. Lin},
	pages = {19212--19223},
	publisher = {Curran Associates, Inc.},
	title = {Transferable Calibration with Lower Bias and Variance in Domain Adaptation},
	url = {https://proceedings.neurips.cc/paper_files/paper/2020/file/df12ecd077efc8c23881028604dbb8cc-Paper.pdf},
	volume = {33},
	year = {2020}
}

@inproceedings{10.24963/ijcai.2023/162,
	author = {Wang, Haoxuan and Yu, Zhiding and Yue, Yisong and Anandkumar, Animashree and Liu, Anqi and Yan, Junchi},
	title = {Learning calibrated uncertainties for domain shift: a distributionally robust learning approach},
	year = {2023},
	isbn = {978-1-956792-03-4},
	url = {https://doi.org/10.24963/ijcai.2023/162},
	doi = {10.24963/ijcai.2023/162},
	booktitle = {Proceedings of the Thirty-Second International Joint Conference on Artificial Intelligence},
	articleno = {162},
	numpages = {10},
	location = {Macao, P.R.China},
	series = {IJCAI '23}
}

@inproceedings{NIPS2010_59c33016,
	author = {Cortes, Corinna and Mansour, Yishay and Mohri, Mehryar},
	booktitle = {Advances in Neural Information Processing Systems},
	editor = {J. Lafferty and C. Williams and J. Shawe-Taylor and R. Zemel and A. Culotta},
	pages = {},
	publisher = {Curran Associates, Inc.},
	title = {Learning Bounds for Importance Weighting},
	url = {https://proceedings.neurips.cc/paper_files/paper/2010/file/59c33016884a62116be975a9bb8257e3-Paper.pdf},
	volume = {23},
	year = {2010}
}

@inproceedings{NEURIPS2022_33d6e648,
	author = {Popordanoska, Teodora and Sayer, Raphael and Blaschko, Matthew},
	booktitle = {Advances in Neural Information Processing Systems},
	editor = {S. Koyejo and S. Mohamed and A. Agarwal and D. Belgrave and K. Cho and A. Oh},
	pages = {7933--7946},
	publisher = {Curran Associates, Inc.},
	title = {A Consistent and Differentiable Lp Canonical Calibration Error Estimator},
	url = {https://proceedings.neurips.cc/paper_files/paper/2022/file/33d6e648ee4fb24acec3a4bbcd4f001e-Paper-Conference.pdf},
	volume = {35},
	year = {2022}
}

@inproceedings{NEURIPS2021_f8905bd3,
	author = {Karandikar, Archit and Cain, Nicholas and Tran, Dustin and Lakshminarayanan, Balaji and Shlens, Jonathon and Mozer, Michael C and Roelofs, Becca},
	booktitle = {Advances in Neural Information Processing Systems},
	editor = {M. Ranzato and A. Beygelzimer and Y. Dauphin and P.S. Liang and J. Wortman Vaughan},
	pages = {29768--29779},
	publisher = {Curran Associates, Inc.},
	title = {Soft Calibration Objectives for Neural Networks},
	url = {https://proceedings.neurips.cc/paper_files/paper/2021/file/f8905bd3df64ace64a68e154ba72f24c-Paper.pdf},
	volume = {34},
	year = {2021}
}

@article{
	bohdal2023metacalibration,
	title={Meta-Calibration: Learning of Model Calibration Using Differentiable Expected Calibration Error},
	author={Ondrej Bohdal and Yongxin Yang and Timothy Hospedales},
	journal={Transactions on Machine Learning Research},
	issn={2835-8856},
	year={2023},
	url={https://openreview.net/forum?id=R2hUure38l},
	note={}
}

@Article{Bolte2014,
	author={Bolte, J{\'e}r{\^o}me
	and Sabach, Shoham
	and Teboulle, Marc},
	title={Proximal alternating linearized minimization for nonconvex and nonsmooth problems},
	journal={Mathematical Programming},
	year={2014},
	month={Aug},
	day={01},
	volume={146},
	number={1},
	pages={459-494},
	issn={1436-4646},
	doi={10.1007/s10107-013-0701-9},
	url={https://doi.org/10.1007/s10107-013-0701-9}
}

@inproceedings{netzer2011reading,
  title={Reading digits in natural images with unsupervised feature learning},
  author={Netzer, Yuval and Wang, Tao and Coates, Adam and Bissacco, Alessandro and Wu, Baolin and Ng, Andrew Y and others},
  booktitle={NIPS workshop on deep learning and unsupervised feature learning},
  volume={2011},
  number={5},
  pages={7},
  year={2011},
  organization={Granada}
}
\bibliographystyle{icml2026}

\newpage
\appendix
\onecolumn
\section*{Appendix}
\section{Top-label Calibration and Class-wise Calibration}
\label{calibration_paradigm}
\begin{definition}
	\textnormal{\textbf{(Top-label Calibration)}} A classifier is perfectly top-label calibrated if the following equation holds:
	\begin{equation}
		P(Y^{*} = \hat Y|\hat S = \hat s) = \hat s, 
	\end{equation}
	where $Y^{*}={{\mathop{\rm argmax}\nolimits} _k}{\{ {Y_k}\} _{1 \le k \le K}}$ is the true class scalar, $\hat Y = {{\mathop{\rm argmax}\nolimits} _k}{\{ {S_k}\} _{1 \le k \le K}}$ is the predicted class, $\hat S = \max {\{ {S_k}\} _{1 \le k \le K}}$ is the confidence score of the predicted class, and $\hat s$ is the observed value on $\hat S$.
	\label{top_label_calibrated}
\end{definition}
\begin{definition}
	\textnormal{\textbf{(Class-wise Calibration)}} A classifier is perfectly class-wise calibrated if the following equation holds:
	\begin{equation}
		P(Y_{k} = 1|S_{k} = s_{k}) = s_{k}, \forall 1 \le k \le K,
		\label{classwise_calibrated_eq}
	\end{equation}
	where $Y_{k}$ is the $k$-th component of the one-hot label $Y$, and $S_{k}$ is the $k$-th component of confidence score vector $S$, and $s_{k}$ is the observed value on $S_{k}$.
	\label{classwise_label_calibrated}
\end{definition}

\section{Proof of Theorem \ref{Effect}}
\label{proof_of_ECC}
\begin{proof}
First, according to Total Probability Theorem, the following holds:
\begin{equation}
\begin{aligned}
P(Y_{k}=1 \mid S) 
&= \int_X P(Y_{k}=1,X \mid S)\,dX \\
&= \int_X P(Y_{k}=1 \mid X,S)\,P(X \mid S)\,dX \\
&= \int_X P(Y_{k}=1 \mid X)\,P(X \mid S)\,dX \\
&= \mathbb{E}_{X\sim P(X\mid S)}[P(Y_{k}=1 \mid X)] .
\end{aligned}
\label{eq_f}
\end{equation}
where the second-to-last equality is because $X$ contains all the information that $S$ can provide. According to the definition of covariate shift, $P_{s}(Y_{k}=1|X)=P_{t}(Y_{k}=1|X)$. Therefore, if $P_{s}(Y_{k}=1|S) = P_{t}(Y_{k}=1|S)$, then:
\begin{equation}
{\mathbb{E}_{X\sim{P_s}(X|S)}}[P(Y_{k}=1|X)] = {\mathbb{E}_{X\sim{P_t}(X|S)}}[P(Y_{k}=1|X)]. 
\end{equation}
where ${P(Y_{k}=1|X)} = {P_{s}(Y_{k}=1|X)} = {P_{t}(Y_{k}=1|X)}$. Conversely, if ${\mathbb{E}_{X \sim {P_s}(X|S)}}[P(Y_{k}=1|X)] = {\mathbb{E}_{X \sim {P_t}(X|S)}}[P(Y_{k}=1|X)]$, it also holds that $P_{s}(Y_{k}=1|S) = P_{t}(Y_{k}=1|S)$.
\end{proof}

\section{Calibration Comparison Between Source and Target Domains}
\label{Source_and_Target}
Typically, a classifier's calibration error in the source domain is significantly lower than that in the target domain because there is no distribution shift that leads to insufficient generalization. For the sake of rigor, we still verified this natural assumption through experiments. Table \ref{Source_and_Target_table} presents the experimental results. We use soft-ECE as the calibration method to calibrate the models in the source domain. All three calibration metrics for different calibration paradigms show that the calibration error in the source domain is significantly lower than that in the target domain. Therefore, even just making the calibration error in the target domain as good as that in the source domain would be a significant improvement.

\begin{table*}[h]
	\centering
	\caption{Comparison of calibration errors between the source and target domains. The subscript $s$ denotes the source domain, while the subscript $t$ denotes the target domain. ResNet-20 is used for the Digit dataset, ResNet-50 for the PACS dataset, and ViT-L for the ImageNet-Sketch dataset.}
	\setlength\tabcolsep{7.0pt}
	\renewcommand{\arraystretch}{1}
	\begin{tabular}{l|cc|cc|cc}
		\toprule
		\textbf{Dataset} & ECE$_{s}$ & ECE$_{t}$ & CwECE$_{s}$ & CwECE$_{t}$& ECE$^{KDE}_{s}$ & ECE$^{KDE}_{t}$ \\
		\midrule
		Digit (USPS + SVHN $\to$ MNIST)&1.54$_{\pm 0.04}$ & 16.2$_{\pm 1.51}$ & 0.39$_{\pm 0.01}$ & 3.14$_{\pm 0.31}$ & 0.39$_{\pm 0.02}$ & 2.97$_{\pm 0.23}$ \\
		\midrule
		PACS (Art + Cartoon + Sketch $\to$ Photo)&3.84$_{\pm 0.23}$ & 22.3$_{\pm 2.16}$ & 0.58$_{\pm 0.01}$ & 7.87$_{\pm 0.31}$ & 0.42$_{\pm 0.04}$ & 7.58$_{\pm 0.37}$ \\
		\midrule
		ImageNet-Sketch (ImageNet $\to$ Sketch)&1.47$_{\pm 0.11}$ & 55.8$_{\pm 4.34}$ & 0.93$_{\pm 0.09}$ & 12.7$_{\pm 0.87}$ & 0.86$_{\pm 0.06}$ & 12.3$_{\pm 0.73}$ \\
		\bottomrule
	\end{tabular}
	\label{Source_and_Target_table}
\end{table*}

\section{Extension of Theorem \ref{Effect}}
\label{Extension}
\begin{theorem}
	\textnormal{\textbf{(Expectation Consistency Condition for Top-label Calibration)}} $P_{s}(Y^{*}=\hat Y|\hat S) = P_{t}(Y^{*}=\hat Y|\hat S)$ if and only if: ${\mathbb{E}_{X \sim {P_s}(X|\hat S)}}[P(Y^{*} = \hat Y|X)] = {\mathbb{E}_{X \sim {P_t}(X|\hat S)}}[P(Y^{*} = \hat Y|X)]$, where ${P(Y^{*} = \hat Y|X)} = {P_{s}(Y^{*} = \hat Y|X)} = {P_{t}(Y^{*} = \hat Y|X)}$.
	\label{Effect_top}
\end{theorem}
\begin{proof}
	First, according to Total Probability Theorem, the following holds:
\begin{equation}
	\begin{split}
		& {P(Y^{*} = \hat Y|\hat S) = \int_X {P(Y^{*} = \hat Y,X|\hat S)dX} } \\ 
		& { = \int_X {P(Y^{*} = \hat Y|X,\hat S)P(X|\hat S)dX}  = \int_X {P(Y^{*} = \hat Y|X)P(X|\hat S)dX} ,} 
	\end{split}
	\label{eq_f_appendix}
\end{equation}
	where the last equality is because $X$ contains all the information that $\hat S$ can provide. According to the definition of covariate shift, $P_{s}(Y^{*}|X)=P_{t}(Y^{*}|X)$. Because the source domain and the target domain share a fixed classifier, $P_{s}(\hat Y|X)=P_{t}(\hat Y|X)$. Then, it holds:
\begin{equation}
	\begin{split}
		& {P_s}(Y^{*},\hat Y|X) = {P_s}(Y^{*}|\hat Y,X){P_s}(\hat Y|X)\\
		& = {P_s}(Y^{*}|X){P_s}(\hat Y|X) = {P_t}(Y^{*}|X){P_t}(\hat Y|X) = {P_t}(Y^{*},\hat Y|X).
	\end{split}
\end{equation}
	where the third equality is because the classifier is fixed, $\hat Y$ is a deterministic function of $X$. Therefore, $P_{s}(Y^{*}=\hat Y|X)=P_{t}(Y^{*}=\hat Y|X)$. According to Eq. \ref{eq_f_appendix}, if $P_{s}(Y^{*}=\hat Y|\hat S) = P_{t}(Y^{*}=\hat Y|\hat S)$, then ${\mathbb{E}_{X \sim {P_s}(X|\hat S)}}[P(Y^{*} = \hat Y|X)] = {\mathbb{E}_{X \sim {P_t}(X|\hat S)}}[P(Y^{*} = \hat Y|X)]$. Conversely, if ${\mathbb{E}_{X \sim {P_s}(X|\hat S)}}[P(Y^{*} = \hat Y|X)] = {\mathbb{E}_{X \sim {P_t}(X|\hat S)}}[P(Y^{*} = \hat Y|X)]$, it also holds that $P_{s}(Y^{*}=\hat Y|\hat S) = P_{t}(Y^{*}=\hat Y|\hat S)$.
\end{proof}
	
\begin{theorem}
	\textnormal{\textbf{(Expectation Consistency Condition for Class-wise Calibration)}} $\forall 1 \le k \le K$, $P_{s}(Y_{k}=1|S_{k}) = P_{t}(Y_{k}=1|S_{k})$ if and only if $\mathbb{E}_{X \sim P_{s}(X|S_{k})}[P(Y_{k}=1|X)] = \mathbb{E}_{X \sim P_{t}(X|S_{k})}[P(Y_{k}=1|X)]$, where $P(Y_{k}=1|X)=P_{s}(Y_{k}=1|X)=P_{t}(Y_{k}=1|X)$.
\label{Effect_classwise}
\end{theorem}

\begin{proof}
By the law of total probability,
\begin{equation}
P(Y_{k}=1|S_{k}) = \int_{X} P(Y_{k}=1,X|S_{k})dX 
= \int_{X} P(Y_{k}=1|X,S_{k}) P(X|S_{k})dX 
= \int_{X} P(Y_{k}=1|X) P(X|S_{k})dX,
\end{equation}
where the last step uses that $X$ contains all information in $S_{k}$ relevant to $Y_{k}$.

Under covariate shift $P_{s}(Y_{k}=1|X)=P_{t}(Y_{k}=1|X)$. Hence:
\begin{equation}
P_{s}(Y_{k}=1|S_{k}) = \int_{X} P(Y_{k}=1|X) P_{s}(X|S_{k}) dX, \quad P_{t}(Y_{k}=1|S_{k}) = \int_{X} P(Y_{k}=1|X) P_{t}(X|S_{k}) dX.
\end{equation}
Therefore $P_{s}(Y_{k}=1|S_{k}) = P_{t}(Y_{k}=1|S_{k})$ iff:
\begin{equation}
\int_{X} P(Y_{k}=1|X) P_{s}(X|S_{k}) dX = \int_{X} P(Y_{k}=1|X) P_{t}(X|S_{k}) dX,
\end{equation}
which is exactly the desired expectation condition.
\end{proof}

\section{Extension of Expectation Consistency Loss}
\label{Loss_Extension}

\subsection{Expectation Consistency Loss for Top-Label Calibration}
Recall the predicted class $\hat Y = {{\mathop{\rm argmax}\nolimits} _k}{\{ {S_k}\} _{1 \le k \le K}}$, its confidence $\hat S = \max {\{ {S_k}\} _{1 \le k \le K}}$, and the true class \(Y^{*}\). Theorem \ref{Effect_top} states that preservation of top-label calibration across domains is equivalent to the expectation consistency condition:
\begin{equation}
\mathbb{E}_{X\sim P_{s}(X|\hat S)}[P(Y^{*}=\hat Y|X)] = \mathbb{E}_{X\sim P_{t}(X|\hat S)}[P(Y^{*}=\hat Y|X)].
\end{equation}
Therefore, \textit{Expectation consistency loss} for top-label calibration can be naturally constructed as:
\begin{equation}
\label{eq:top_ecl}
L_{ecl}^{top} = \mathbb{E}_{P_{t}(\hat S)} \Big| \mathbb{E}_{P_{s}(X|\hat S)} P(Y^{*}=\hat Y|X) - \mathbb{E}_{P_{t}(X|\hat S)} P(Y^{*}=\hat Y|X) \Big|.
\end{equation}
To estimate $P(Y^{*}=\hat Y|X)$ in practice, we train a binary classifier where the label is $1_{Y^{*}=\hat Y}$ and the input data is $X$. This binary classifier can be added to the original classifier as a classification head and trained end-to-end with the original classifier (freeze the backbone when training this classification head). Optionally, this binary classification head can also be calibrated on the source domain to obtain a more reliable estimate of $P(Y^{*}=\hat Y|X)$.

\subsection{Expectation Consistency Loss for Class-wise Calibration}
For class-wise calibration, each coordinate \(S_{k}\) must match \(P(Y_{k}=1|S_{k})\). Theorem \ref{Effect_classwise} implies expectation consistency per class:
\begin{equation}
	\mathbb{E}_{X\sim P_{s}(X|S_{k})}[P(Y_{k}=1|X)] = \mathbb{E}_{X\sim P_{t}(X|S_{k})}[P(Y_{k}=1|X)], \quad \forall k \in \{1,\dots,K\}.
\end{equation}
Therefore, \textit{Expectation consistency loss} for class-wise calibration can be naturally constructed as:
\begin{equation}
	\label{eq:cw_ecl}
	L_{\mathrm{ecl}}^{\mathrm{cw}} = \sum_{k=1}^{K} \Big[ \mathbb{E}_{P_{t}(S_{k})} \Big| \mathbb{E}_{P_{s}(X|S_{k})} P(Y_{k}=1|X) - \mathbb{E}_{P_{t}(X|S_{k})} P(Y_{k}=1|X) \Big|\Big].
\end{equation}
To estimate $P(Y_{k}=1|X)$ in practice, we train an additional classification head on the original classifier's backbone, where the label is $Y_{k}$ (the $k$-th component of the one-hot encoded label) and the input data is $X$. This classification head can be trained end-to-end with the original classifier (freeze the backbone when training this classification head). Optionally, this classification head can also be calibrated on the source domain.

\section{Extensions on Empirical Calculation and Differentiability}
\label{appendix: Empirical Calculation and Differentiability}

\textbf{Empirical Calculation and Differentiability for Top-label Calibration: } For top-label calibration, \textit{Expectation Consistency Loss} can be empirically estimated using confidence binning and Monte Carlo sampling:
\begin{equation}
	\begin{dcases}
			{ {{\hat L_{ecl}^{\mathrm{top}}}} = \sum\limits_{j = 1}^B {\frac{{\sharp b_j^{(t)}}}{{\sharp {D_t}}} \left\lVert {{{\hat {\mathbb{E}}}_{s,j}} - {{\hat {\mathbb{E}}}_{t,j}}} \right\rVert} ,}\\
			{{\hat{\mathbb{E}}_{s,j}} = \frac{1}{{\sharp D_s^{(j)}}}\sum\limits_{x \in D_s^{(j)}} {\hat P(Y^{*}=\hat Y|X = x)} ,}\\
			{{\hat{\mathbb{E}}_{t,j}} = \frac{1}{{\sharp D_t^{(j)}}}\sum\limits_{x \in D_t^{(j)}} {\hat P(Y^{*}=\hat Y|X = x)} ,}
	\end{dcases}
	\label{top_label_estimation}
\end{equation}
where $B$ represents the number of bins, $b_j^{(t)}$ represents the $j$-th bin in the target domain, $\sharp b_{j}^{(t)}$ represents sample size of $b_{j}^{(t)}$, $\sharp D_{t}$ represents sample size of $D_{t}$, ${D_s^{(j)}}$ represents the level set of $b_j^{(t)}$ in the source domain, ${D_t^{(j)}}$ represents the level set of $b_j^{(t)}$ in the target domain, and ${{{\hat P}}(Y^{*} = \hat Y|X = x)}$ represents the observation of ${{{P}}(Y^{*} = \hat Y|X)}$. For differentiability, introduce anchors $a_j=(2j-1)/(2B)$ and weights $\omega_{ij}=\exp(- (\hat S^{(i)}-a_j)^2/\tau)/\sum_{r} \exp(- (\hat S^{(i)}-a_r)^2/\tau)$ with temperature $\tau>0$. Denoting $p^{(i)}=P(Y^{*}=\hat Y|X_{i})$ as the output of the binary classification head (as described in Section \ref{Loss_Extension}), we obtain for each bin $j$ and domain $d\in\{s,t\}$:
\begin{equation}
	\hat{\mathbb{E}}_{d,j}=\frac{\sum_{i} \omega_{ij}^{d} p^{(i)}}{\sum_{i} \omega_{ij}^{d}+\varepsilon}.
\end{equation}
Therefore, the differentiable ECL for top-label calibration is $\hat L_{ecl}^{\text{top}}=\sum_{j=1}^B w_j \lVert\hat{\mathbb{E}}_{s,j}-\hat{\mathbb{E}}_{t,j}\rVert$, where $w_j=\frac{\sum_{i} \omega^{t}_{ij}}{\sum_{r} \sum_{i} \omega^{t}_{ir}}$.

\textbf{Empirical Calculation and Differentiability for Class-wise Calibration: } For class-wise calibration, \textit{Expectation Consistency Loss} can be empirically estimated using confidence binning and Monte Carlo sampling:
\begin{equation}
\begin{dcases}
 \hat L_{ecl}^{\mathrm{cw}} 
= \sum_{k=1}^{K} \sum_{j = 1}^{B} \frac{\sharp b_{k,j}^{(t)}}{\sharp D_t} \, \Big| \hat{\mathbb{E}}_{s,k,j} - \hat{\mathbb{E}}_{t,k,j} \Big|, \\
 \hat{\mathbb{E}}_{s,k,j} 
= \frac{1}{\sharp D_{s,k}^{(j)}} \sum_{x \in D_{s,k}^{(j)}} \hat P(Y_{k}=1\,|\,X=x), \\
 \hat{\mathbb{E}}_{t,k,j} 
= \frac{1}{\sharp D_{t,k}^{(j)}} \sum_{x \in D_{t,k}^{(j)}} \hat P(Y_{k}=1\,|\,X=x).
\end{dcases}
\label{cw_estimation}
\end{equation}
where $B$ is the number of bins per class, $b_{k,j}^{(t)}$ is the $j$-th bin for class-$k$ on the target domain (formed by binning $S_{k}$), $\sharp b_{k,j}^{(t)}$ is its size, $\sharp D_{t}$ is the target sample size, and $D_{s,k}^{(j)}$, $D_{t,k}^{(j)}$ are the level sets of $b_{k,j}^{(t)}$ on source/target domains, respectively. For differentiability, let anchors $a_j = \frac{2j-1}{2B}$ for $j=1,\dots,B$, and define soft weights for a sample $i$ with confidence $S_{k}^{(i)}$:
\begin{equation}
\omega_{k,ij} = \frac{\exp\!\big(- (S_{k}^{(i)}-a_j)^2/\tau\big)}{\sum_{r=1}^{B} \exp\!\big(- (S_{k}^{(i)}-a_r)^2/\tau\big)}, \quad \tau>0.
\end{equation}
For domain $d\in\{s,t\}$, define
\begin{equation}
\hat{\mathbb{E}}_{d,k,j} = \frac{\sum_{i} \omega_{k,ij}^{d} \, p_{k}^{(i)}}{\sum_{i} \omega_{k,ij}^{d} + \varepsilon}, \quad n_{k,j}^{d} = \sum_{i} \omega_{k,ij}^{d}, \quad p_{k}^{(i)} = P(Y_{k}=1|X_{i}),
\end{equation}
with stabilizer $\varepsilon>0$. The differentiable class-wise ECL becomes
\begin{equation}
\hat L_{ecl}^{\mathrm{cw}} = \sum_{k=1}^{K} \sum_{j=1}^{B} w_{k,j} \, \Big| \hat{\mathbb{E}}_{s,k,j} - \hat{\mathbb{E}}_{t,k,j} \Big|, \quad w_{k,j} = \frac{n_{k,j}^{t}}{\sum_{r=1}^{B} n_{k,r}^{t}}.
\label{Differentiable_ECL_CW}
\end{equation}

\section{Proof of Theorem \ref{thm:sample_complexity}}
\label{proof_sample_complexity}
\begin{proof}
	For each bin $j$, define random variables $Z_{s,j} = \lVert \hat{\mathbb{E}}_{s,j} - \mathbb{E}_{P_{s}(X|S)} P(Y|X) \rVert$ and $Z_{t,j} = \lVert \hat{\mathbb{E}}_{t,j} - \mathbb{E}_{P_{t}(X|S)} P(Y|X)			 \rVert$. By the triangle inequality,
	\begin{equation}
		\big|\hat L_{ecl} - L_{ecl}\big| \le \sum_{j=1}^{B} w_j \, (Z_{s,j} + Z_{t,j}).
	\end{equation}
	Using Hoeffding's inequality and a union bound over bins and classes, there exist absolute constants $C_1,C_2>0$ such that, with probability at least $1-\delta$,
	\begin{equation}
		Z_{s,j} \le C_1 \sqrt{\frac{K \log(2BK/\delta)}{n_{s,j}}}, \quad Z_{t,j} \le C_2 \sqrt{\frac{K \log(2BK/\delta)}{n_{t,j}}}, \quad \forall j=1,\dots,B.
	\end{equation}
	Combining these bounds gives the desired result.
\end{proof}

\section{Proof of Theorem \ref{mini_batch_thm}}
\label{proof_mini_batch_thm}
This proof proceeds in two steps. First we show that Eq. \ref{mini_batch_eq} is an auxiliary-variable reformulation of Eq. \ref{Differentiable_ECL}: minimizing the auxiliary variables $u_j^{s},u_j^{t}$ in Eq. \ref{mini_batch_eq} recovers Eq. \ref{Differentiable_ECL}. Second we show that, under the auxiliary-variable formulation, the mini-batch gradient is an unbiased estimator of the full-sample gradient.

\paragraph{Equivalence between Eq. \ref{mini_batch_eq} and Eq. \ref{Differentiable_ECL}.} Fix $\theta$ and consider minimizing the right-hand side of Eq. \ref{mini_batch_eq} with respect to the auxiliary vectors $u_j^s,u_j^t$ for each bin $j$. The terms that depend on $u_j^s,u_j^t$ are
\begin{equation*}
G_j(u_j^s,u_j^t)= w_j\|u_j^s-u_j^t\| + \sum_{i\in D_s} \omega_{i,j}^s \|u_j^s-p_i(\theta)\|^2 + \sum_{i\in D_t} \omega_{i,j}^t \|u_j^t-p_i(\theta)\|^2.
\end{equation*}
Define the soft counts and weighted empirical means
\begin{equation*}
n^{s}_j=\sum_{i\in D_s}\omega_{i,j}^s,\quad n^{t}_j=\sum_{i\in D_t}\omega_{i,j}^t,
\quad \hat{\mathbb{E}}_{s,j}=\frac{1}{n^{s}_j}\sum_{i\in D_s}\omega_{i,j}^s p^{(i)}(\theta),\quad
\hat{\mathbb{E}}_{t,j}=\frac{1}{n^{t}_j}\sum_{i\in D_t}\omega_{i,j}^t p^{(i)}(\theta).
\end{equation*}
The quadratic terms are strongly convex in $u_j^s,u_j^t$, so $G_j$ has a unique minimizer. Taking (sub)gradients $w.r.t.$ $u_j^s,u_j^t$ and setting them to zero yields
\begin{equation*}
2 n^{s}_j(u_j^s-\hat{\mathbb{E}}_{s,j}) + w_j g_j = 0,\qquad 2 n^{t}_j(u_j^t-\hat{\mathbb{E}}_{t,j}) - w_j g_j = 0,
\end{equation*}
where $g_j$ is any subgradient of the norm at $u_j^s-u_j^t$ (a unit vector when the difference is nonzero). Eliminating $g_j$ gives
\begin{equation*}
u_j^s = \hat{\mathbb{E}}_{s,j} - \frac{w_j}{2 n^{s}_j} g_j,\qquad
u_j^t = \hat{\mathbb{E}}_{t,j} + \frac{w_j}{2 n^{t}_j} g_j.
\end{equation*}
When the quadratic penalty terms are minimized (forcing the auxiliary variables to their weighted empirical means), the correction terms vanish and
\begin{equation*}
u_j^s\to\hat{\mathbb{E}}_{s,j},\qquad u_j^t\to\hat{\mathbb{E}}_{t,j}.
\end{equation*}
Substituting these optimal auxiliary values back into Eq. \ref{mini_batch_eq} yields
\begin{equation*}
\sum_{j=1}^B w_j \big\|\hat{\mathbb{E}}_{s,j}-\hat{\mathbb{E}}_{t,j}\big\|,
\end{equation*}
which is exactly Eq. \ref{Differentiable_ECL}. Hence Eq.~\ref{mini_batch_eq} is asymptotically equivalent to Eq.~\ref{Differentiable_ECL}, with an $O(w_j/n_j^d)$ gap (from the subgradient penalties $\frac{w_j}{2n_j^s}g_j$ and $\frac{w_j}{2n_j^t}g_j$) that vanishes as $n_j^s,n_j^t\to\infty$.

\paragraph{Unbiasedness of the mini-batch gradient.} We will first prove that Eq. \ref{Differentiable_ECL} produces a biased gradient estimate on mini-batches, and then prove that Eq. \ref{mini_batch_eq} produces an unbiased gradient estimate. 

Write the differentiable ECL (Eq. \ref{Differentiable_ECL}) as
\begin{equation*}
\hat L_{ecl}(\theta)=\sum_{j=1}^B w_j \big\|\hat{\mathbb{E}}_{s,j}-\hat{\mathbb{E}}_{t,j}\big\|.
\end{equation*}
For notational clarity and for an arbitrary norm $\|\cdot\|$ introduce a subgradient selection
\begin{equation*}
g_j \in \partial\|\hat{\mathbb{E}}_{s,j}-\hat{\mathbb{E}}_{t,j}\|\quad(\text{any choice when the difference is nonzero}).
\end{equation*}
Using the chain rule for a general norm we obtain the full-data gradient
\begin{equation}
\nabla_\theta \hat L_{ecl}(\theta)=\sum_{j=1}^B w_j \left\langle g_j,\; \nabla_\theta\hat{\mathbb{E}}_{s,j}-\nabla_\theta\hat{\mathbb{E}}_{t,j} \right\rangle,
\label{grad_full}
\end{equation}
where, for example, the full-data weighted gradient average is
\begin{equation*}
\nabla_\theta\hat{\mathbb{E}}_{s,j}=\frac{1}{n^{s}_j}\sum_{i\in D_s} \omega_{i,j}^s \, \nabla_\theta p^{(i)}(\theta).
\end{equation*}

Now consider computing the same expression on a random mini-batch. Let
\(\hat{\mathbb{E}}_{s,j}^{\rm m},\hat{\mathbb{E}}_{t,j}^{\rm m}\) be the per-bin weighted means computed from the current mini-batches and choose a measurable subgradient selection
\(g_j^{\rm m}\in\partial\|\hat{\mathbb{E}}_{s,j}^{\rm m}-\hat{\mathbb{E}}_{t,j}^{\rm m}\|\).
The mini-batch gradient contribution for bin $j$ (when using Eq. \ref{Differentiable_ECL} directly on the mini-batch) equals
\begin{equation*}
G_j^{\rm m} = w_j \left\langle g_j^{\rm m},\; \nabla_\theta \hat{\mathbb{E}}_{s,j}^{\rm m} - \nabla_\theta \hat{\mathbb{E}}_{t,j}^{\rm m} \right\rangle.
\end{equation*}
Taking expectation over the random mini-batch sampling (the indices in the sums) and using linearity gives
\begin{equation}
\mathbb{E}[G_j^{\rm m}] = w_j \left( \mathbb{E}[g_j^{\rm m}]^\top \, \mathbb{E}[\nabla_\theta \hat{\mathbb{E}}_{s,j}^{\rm m} - \nabla_\theta \hat{\mathbb{E}}_{t,j}^{\rm m}] \, + \, \mathrm{Cov}\big(g_j^{\rm m},\,\nabla_\theta \hat{\mathbb{E}}_{s,j}^{\rm m} - \nabla_\theta \hat{\mathbb{E}}_{t,j}^{\rm m}\big) \right),
\label{bias_term}
\end{equation}
where the covariance denotes the cross-covariance between the components of the subgradient vector $g_j^{\rm m}$ and the gradient estimator. The covariance need not vanish because $g_j^{\rm m}$ is a nonlinear (sub)differential selection of the same mini-batch samples that produce the per-sample gradients; hence in general
\begin{equation*}
\mathbb{E}[G_j^{\rm m}] \ne w_j g_j^\top \big(\nabla_\theta\hat{\mathbb{E}}_{s,j} - \nabla_\theta\hat{\mathbb{E}}_{t,j}\big).
\end{equation*}
This equality would hold only if $g_j^{\rm m}$ were (in expectation) equal to $g_j$ and uncorrelated with the mini-batch gradient estimator — a condition that generally fails because of the nonlinear subgradient selection.

Eq. \ref{mini_batch_eq} (Eq.13) remedies this issue by introducing auxiliary variables $u_j^{s},u_j^{t}$. Concretely, let us set the auxiliaries to the full-data weighted means (functions of $\theta$ but independent of the current mini-batch indices):
\begin{equation*}
u_j^{s,\mathrm{full}}:=\hat{\mathbb{E}}_{s,j}=\frac{1}{n^{s}_j}\sum_{i\in D_s}\omega_{i,j}^s p^{(i)}(\theta),\qquad
u_j^{t,\mathrm{full}}:=\hat{\mathbb{E}}_{t,j}=\frac{1}{n^{t}_j}\sum_{i\in D_t}\omega_{i,j}^t p^{(i)}(\theta).
\end{equation*}
Define the fixed unit vector
\begin{equation*}
v_j^{\mathrm{full}}:=\frac{u_j^{s,\mathrm{full}}-u_j^{t,\mathrm{full}}}{\|u_j^{s,\mathrm{full}}-u_j^{t,\mathrm{full}}\|}.
\end{equation*}
If we compute the mini-batch gradient of Eq. \ref{mini_batch_eq} while treating $u_j^{d}=u_j^{d,\mathrm{full}}$ as fixed (i.e. independent of the current mini-batch samples), the bin-$j$ contribution equals
\begin{equation*}
	\tilde G_j^{\rm m} = w_j \left\langle v_j^{\mathrm{full}},\; \frac{1}{|D_s^{\rm m}|}\sum_{i\in D_s^{\rm m}} \omega_{i,j}^s \, \nabla_\theta p^{(i)}(\theta) - \frac{1}{|D_t^{\rm m}|}\sum_{i\in D_t^{\rm m}} \omega_{i,j}^t \, \nabla_\theta p^{(i)}(\theta) \right\rangle.
\end{equation*}
Taking expectation over the random mini-batch sampling (the indices in the sums) and using linearity gives
\begin{equation*}
\mathbb{E}[\tilde G_j^{\rm m}] = w_j \left\langle v_j^{\mathrm{full}},\; \frac{1}{n^{s}_j}\sum_{i\in D_s} \omega_{i,j}^s \, \nabla_\theta p^{(i)}(\theta) - \frac{1}{n^{t}_j}\sum_{i\in D_t} \omega_{i,j}^t \, \nabla_\theta p^{(i)}(\theta) \right\rangle.
\end{equation*}
The right-hand side is exactly the full-data bin-$j$ term in Eq. \ref{grad_full}; summing over $j$ yields
\begin{equation*}
\mathbb{E}\Big[\sum_{j=1}^B \tilde G_j^{\rm m}\Big] = \nabla_\theta \hat L_{ecl}(\theta).
\end{equation*}
Thus, when Eq. \ref{mini_batch_eq} is used with auxiliaries taken from an estimate independent of the current mini-batch (e.g. full-data means, a large buffer, or a slow running average), the mini-batch gradient is an unbiased estimator of the full-sample gradient.

\begin{minipage}[t]{0.48\linewidth}
\begin{algorithm}[H]
	\caption{Top-label ECL Mini-Batch.}
	\begin{algorithmic}[1]
	\STATE \textbf{Input:}
		\STATE \quad bins $j=1\ldots B$, hyperparameters $\lambda,\alpha_{\text{ema}},N_{\text{prox}}$;
		\STATE \quad $u_j^s = 0 \in \mathbb{R}, \forall j$; $u_j^t = 0 \in \mathbb{R}, \forall j$;
	\FOR{each iteration}
		\STATE Sample mini-batches $D_s^m,D_t^m$; 
		\STATE Compute weights $\omega_{ij}^s,\omega_{ij}^t$;
		\STATE $n_{s,j} \leftarrow \sum_{i\in D_s^m}\omega_{ij}^s$; $n_{t,j} \leftarrow \sum_{i\in D_t^m}\omega_{ij}^t$;
		\STATE $m_{s,j} \leftarrow \sum_{i\in D_s^m}\omega_{ij}^s P(Y^{*}=\hat Y| X=x_i)$;
		\STATE $m_{t,j} \leftarrow \sum_{i\in D_t^m}\omega_{ij}^t P(Y^{*}=\hat Y| X=x_i)$;
		\STATE $w_j \leftarrow n_{t,j} / \sum_{r=1}^{B} n_{t,r}$;
		\STATE $L_{\text{ecl}} \leftarrow 0$;
		\FOR{each bin $j$}
			\STATE $u_s,u_t\leftarrow$ cached $u_j^s,u_j^t$
			\FOR{$i=1$ to $N_{\text{prox}}$}
				\STATE $v_s\leftarrow (m_{s,j}/n_{s,j})-u_t$, $\tau_s=\dfrac{w_j}{2 n_{s,j}}$
				\STATE $u_s\leftarrow u_t + \mathrm{shrink}(v_s,\tau_s)$
				\STATE $v_t\leftarrow (m_{t,j}/n_{t,j})-u_s$, $\tau_t=\dfrac{w_j}{2 n_{t,j}}$
				\STATE $u_t\leftarrow u_s + \mathrm{shrink}(v_t,\tau_t)$
			\ENDFOR
			\STATE $\tilde u_j^s,\tilde u_j^t\leftarrow u_s.{\rm detach}(),\;u_t.{\rm detach}()$
			\STATE $u_j^s\leftarrow(1-\alpha_{\text{ema}})u_j^s+\alpha_{\text{ema}}\tilde u_j^s$
			\STATE $u_j^t\leftarrow(1-\alpha_{\text{ema}})u_j^t+\alpha_{\text{ema}}\tilde u_j^t$
			\STATE $L_{\text{ecl}} \mathrel{+}= \sum\limits_{i\in D_s^m}\omega_{ij}^s\|\tilde u_j^s-P(Y^{*}=\hat Y| X=x_i)\|^2$
			\STATE $L_{\text{ecl}} \mathrel{+}=  \sum\limits_{i\in D_t^m}\omega_{ij}^t\|\tilde u_j^t-P(Y^{*}=\hat Y| X=x_i)\|^2$
		\ENDFOR
		\STATE Compute the cross-entropy loss $L_{\text{ce}}$
		\STATE Backpropagate $L_{\text{ce}} + \lambda L_{\text{ecl}}$ and update $\theta$
	\ENDFOR
	\STATE \textbf{Return:} $\theta$
	\end{algorithmic}
	\label{alg:top_label_mb}
\end{algorithm}
\end{minipage}
\hfill
\begin{minipage}[t]{0.48\linewidth}
\begin{algorithm}[H]
	\caption{Class-wise ECL Mini-Batch.}
	\begin{algorithmic}[1]
	\STATE \textbf{Input:}
		\STATE \quad bins $j=1\ldots B$, hyperparameters $\lambda,\alpha_{\text{ema}},N_{\text{prox}}$;
		\STATE \quad $u_{k,j}^s = 0 \in \mathbb{R}, \forall k,j$; $u_{k,j}^t = 0 \in \mathbb{R}, \forall k,j$;
	\FOR{each iteration}
		\STATE Sample mini-batches $D_s^m,D_t^m$; $L_{\text{ecl}} \leftarrow 0$;
		\FOR{each class $k=1$ to $K$}
			\STATE Compute weights $\omega_{k,ij}^s,\omega_{k,ij}^t$;
            \STATE $n_{s,j} \leftarrow \sum_{i\in D_s^m}\omega_{k,ij}^s$; $n_{t,j} \leftarrow \sum_{i\in D_t^m}\omega_{k,ij}^t$;
            \STATE $m_{s,j} \leftarrow \sum_{i\in D_s^m}\omega_{k,ij}^s p^{(i)}_{k}(\theta)$;
		    \STATE $m_{t,j} \leftarrow \sum_{i\in D_t^m}\omega_{k,ij}^t p^{(i)}_{k}(\theta)$;
		    \STATE $w_{k,j} \leftarrow n_{t,j} / \sum_{r=1}^{B} n_{t,r}$;
			\FOR{each bin $j$}
				\STATE $u_s,u_t\leftarrow$ cached $u_{k,j}^s,u_{k,j}^t$
				\FOR{$i=1$ to $N_{\text{prox}}$}
					\STATE $v_s\leftarrow (m_{s,j}/n_{s,j})-u_t$, $\tau_s=\dfrac{w_{k,j}}{2 n_{s,j}}$
					\STATE $u_s\leftarrow u_t + \mathrm{shrink}(v_s,\tau_s)$
					\STATE $v_t\leftarrow (m_{t,j}/n_{t,j})-u_s$, $\tau_t=\dfrac{w_{k,j}}{2 n_{t,j}}$
					\STATE $u_t\leftarrow u_s + \mathrm{shrink}(v_t,\tau_t)$
				\ENDFOR
				\STATE $\tilde u_{k,j}^s,\tilde u_{k,j}^t\leftarrow u_s.{\rm detach}(),\;u_t.{\rm detach}()$
				\STATE $u_{k,j}^s\leftarrow(1-\alpha_{\text{ema}})u_{k,j}^s+\alpha_{\text{ema}}\tilde u_{k,j}^s$
				\STATE $u_{k,j}^t\leftarrow(1-\alpha_{\text{ema}})u_{k,j}^t+\alpha_{\text{ema}}\tilde u_{k,j}^t$
				\STATE $L_{\text{ecl}} \mathrel{+}= \sum\limits_{i\in D_s^m}\omega_{k,ij}^s\|\tilde u_{k,j}^s-p^{(i)}_{k}(\theta)\|^2$
				\STATE $L_{\text{ecl}} \mathrel{+}=  \sum\limits_{i\in D_t^m}\omega_{k,ij}^t\|\tilde u_{k,j}^t-p^{(i)}_{k}(\theta)\|^2$
			\ENDFOR
		\ENDFOR
		\STATE Compute the cross-entropy loss $L_{\text{ce}}$
		\STATE Backpropagate $L_{\text{ce}} + \lambda L_{\text{ecl}}$ and update $\theta$
	\ENDFOR
	\STATE \textbf{Return:} $\theta$
	\end{algorithmic}
	\label{alg:class_wise_mb}
\end{algorithm}
\end{minipage}
\section{Extension of ECL Mini-Batch Training}
\label{Algorithom_Appendix}
Algorithm \ref{alg:mini_batch_updates} details the ECL mini-batch training for canonical calibration. Here, we present the analogous algorithms for top-label calibration (Algorithm \ref{alg:top_label_mb}) and class-wise calibration (Algorithm \ref{alg:class_wise_mb}). They employ the same auxiliary variable strategy to resolve the bias in mini-batch gradients. In Algorithm \ref{alg:top_label_mb}, $P(Y^{*}=\hat Y| X=x)$ can be obtained by training a binary
classifier where the label is $1_{Y^{*}=\hat Y}$ and the input data is $X$. Moreover, this binary classifier does not need to be trained separately. It can be added to the original classifier as a classification head and trained end-to-end with the original classifier (freeze the backbone when training this classification head).
\section{Results}
\label{results_in_appendix}
\textbf{Other experimental settings:} 
The batch size in the experiment is uniformly set to 100. Adam optimizer with a learning rate of 0.001 is used to train the classifier for 100 epochs. All experiments were conducted on Intel$^{\circledR}$ Core$^{TM}$ I7-10700 CPU with 3.70GHz and 125.5GB memory, 10 NVIDIA GeForce RTX 3090 graphics cards (each with 24GB of video memory), Ubuntu 20.04.3 LTS, Python 3.11.11, and Torch 2.4.1+cu118. We calibrate the classification head used to estimate $P(Y|X)$ (or $P(Y^{*}=\hat Y|X)$ for top-label calibration) on the source domain using Soft-ECE loss. This classification head has the same network structure as the classification head in the original classifier, and uses the same hyperparameters during training. All images in the digit recognition dataset were standardized to 3-channel RGB format and resized to a resolution of 28$\times$28 pixels. All images in PACS and ImageNet-Sketch were standardized to 3-channel RGB format and resized to a resolution of 224$\times$224 pixels.

\subsection{Results on Simulated Covariate Shifts Data}
\label{Simulated_appendix}
Figure~\ref{fig:uniform} shows the calibration results under a uniformly distributed covariate shift, complementing the normally distributed case in Figure~\ref{fig_before_abstract}. Consistent with the normal case, ECL achieves the lowest calibration error across all three paradigms.
\begin{figure*}[t]
	\begin{center}
		\centerline{\includegraphics[width=1.0\textwidth]{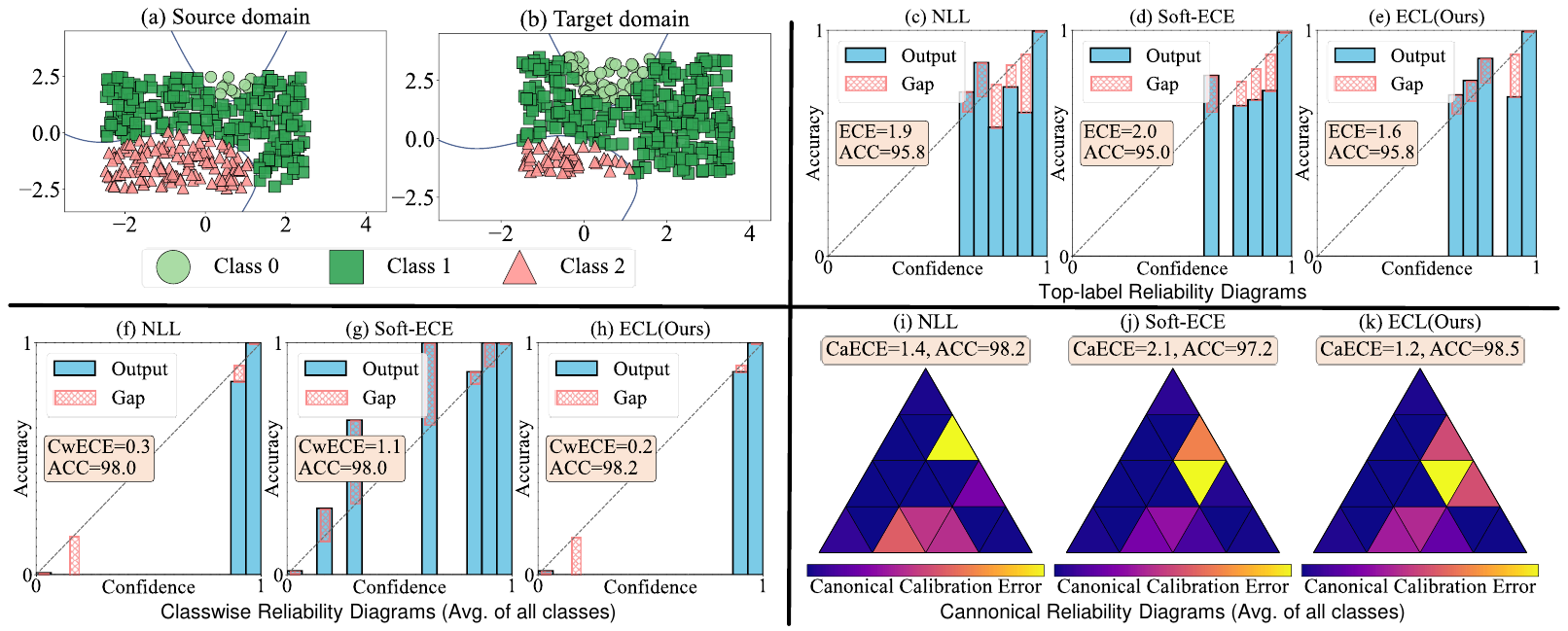}}
		\caption{The calibration results are presented using simulated data under a uniformly distributed covariate shift. From the calibration metric on the target domain and the reliability diagram of the calibrated classifier, ECL achieves the smallest calibration error.}
		\label{fig:uniform}
	\end{center}
	\vskip -0.2in
\end{figure*}

\subsection{Results for Top-label Calibration}
\label{Results_top_label}
Table \ref{Calibration_comparison_on_PACS_ImageNet_transformed} details the top-label calibration performance on the PACS and ImageNet-Sketch datasets. Several key observations can be drawn regarding the effectiveness of ECL. First, regarding robustness to large shifts, on the ImageNet-Sketch dataset—which presents a severe distribution shift (source ImageNet vs. target Sketch)—uncalibrated models exhibit extreme ECE values exceeding 55\%. ECL substantially reduces these errors (often to below 15\% across the tested architectures), demonstrating its capability to handle substantial domain gaps. Second, while PseudoCal serves as a strong baseline, ECL is generally competitive and frequently achieves a lower ECE. For instance, in the PACS ($\to$ Cartoon) task using Wide-Res50, ECL achieves an ECE of 7.61\%, outperforming PseudoCal (16.24\%) and improving upon DRL (8.36\%). Finally, the method remains effective across diverse architectures, from standard CNNs (ResNet, DenseNet) to Vision Transformers (ViT-L), suggesting that the Expectation Consistency condition captures a model-agnostic principle.

\begin{table*}[t]
    \centering
    \footnotesize
    \caption{ECE (\%) for top-label calibration on PACS and ImageNet-Sketch datasets. The reported results represent the mean and standard deviation derived from ten runs.}
    \setlength\tabcolsep{1.5pt}
    \renewcommand{\arraystretch}{1.}
    \begin{tabular}{clccccccccc|cc}
        \toprule
        \multicolumn{2}{c}{\multirow{2}{*}{\textbf{Datasets}}}&\multicolumn{9}{c|}{\textbf{ECE} $\bm{\downarrow}$}& \\
        & &Uncal&Soft-ECE&DECE&KDE&TS&TransCal&DRL&PseudoCal&\cellcolor{color1}ECL (Ours)&\multirow{-2}{*}{\textbf{Oracle} $\bm{\downarrow}$}&\multirow{-2}{*}{\textbf{$\Delta$ACC(\%)}}\\
        \midrule
        \multirow{16}{*}{\rotatebox{90}{\textbf{PACS}}}&\textbf{$\to$ \textit{Photo}}& & & & & & & & & & \\
        &ResNet50&22.3$_{\pm 2.16}$&22.1$_{\pm 1.83}$&22.8$_{\pm 1.99}$&21.8$_{\pm 1.72}$&20.9$_{\pm 1.53}$&22.2$_{\pm 1.66}$&9.02$_{\pm 0.57}$&7.33$_{\pm 0.44}$&\cellcolor{color1}6.87$_{\pm 0.34}$&3.84$_{\pm 0.23}$&+0.72$_{\pm 0.17}$\\
        &DenseNet121&9.78$_{\pm 0.96}$&9.88$_{\pm 0.91}$&9.54$_{\pm 0.91}$&10.2$_{\pm 0.86}$&9.63$_{\pm 0.69}$&9.31$_{\pm 0.63}$&7.91$_{\pm 0.41}$&6.61$_{\pm 0.54}$&\cellcolor{color1}5.96$_{\pm 0.27}$&1.84$_{\pm 0.13}$&-0.83$_{\pm 0.23}$\\
        &Wide-Res50&16.9$_{\pm 1.42}$&17.2$_{\pm 1.24}$&17.8$_{\pm 1.39}$&16.8$_{\pm 1.18}$&16.2$_{\pm 1.37}$&7.27$_{\pm 0.47}$&4.39$_{\pm 0.42}$&2.83$_{\pm 0.11}$&\cellcolor{color1}2.68$_{\pm 0.33}$&1.59$_{\pm 0.17}$&+0.69$_{\pm 0.22}$\\
        \cmidrule{2-13}
        &\textbf{$\to$ \textit{Art}}& & & & & & & & & & & \\
        &ResNet50&33.1$_{\pm 3.24}$&32.1$_{\pm 2.97}$&33.2$_{\pm 3.11}$&31.6$_{\pm 3.06}$&31.9$_{\pm 3.24}$&17.1$_{\pm 1.51}$&17.1$_{\pm 1.14}$&7.88$_{\pm 0.78}$&\cellcolor{color1}7.22$_{\pm 0.53}$&2.12$_{\pm 0.08}$&-1.24$_{\pm 0.41}$\\
        &DenseNet121&23.2$_{\pm 2.04}$&22.8$_{\pm 1.94}$&23.6$_{\pm 2.14}$&23.1$_{\pm 1.86}$&22.7$_{\pm 1.96}$&22.4$_{\pm 1.88}$&6.16$_{\pm 0.59}$&9.94$_{\pm 0.74}$&\cellcolor{color1}5.89$_{\pm 0.36}$&2.24$_{\pm 0.21}$&+1.06$_{\pm 0.39}$\\
        &Wide-Res50&29.9$_{\pm 2.78}$&29.3$_{\pm 2.57}$&30.4$_{\pm 2.64}$&29.7$_{\pm 2.42}$&30.1$_{\pm 2.53}$&16.1$_{\pm 1.23}$&15.8$_{\pm 1.42}$&8.43$_{\pm 0.67}$&\cellcolor{color1}7.97$_{\pm 0.52}$&3.14$_{\pm 0.24}$&-0.96$_{\pm 0.29}$\\
        \cmidrule{2-13}
        &\textbf{$\to$ \textit{Cartoon}}& & & & & & & & & & & \\
        &ResNet50&25.1$_{\pm 2.26}$&25.1$_{\pm 2.08}$&25.3$_{\pm 2.42}$&24.9$_{\pm 2.07}$&24.8$_{\pm 1.91}$&25.2$_{\pm 2.39}$&6.69$_{\pm 0.48}$&5.71$_{\pm 0.43}$&\cellcolor{color1}5.46$_{\pm 0.43}$&2.73$_{\pm 0.26}$&+0.56$_{\pm 0.12}$\\
        &DenseNet121&18.4$_{\pm 1.48}$&18.7$_{\pm 1.36}$&17.8$_{\pm 1.56}$&18.4$_{\pm 1.44}$&18.3$_{\pm 1.73}$&11.3$_{\pm 0.91}$&10.9$_{\pm 1.21}$&2.21$_{\pm 0.09}$&\cellcolor{color1}2.04$_{\pm 0.16}$&2.04$_{\pm 0.18}$&-0.74$_{\pm 0.22}$\\
        &Wide-Res50&25.4$_{\pm 1.98}$&24.9$_{\pm 1.88}$&25.9$_{\pm 2.06}$&25.6$_{\pm 1.72}$&25.2$_{\pm 1.76}$&23.7$_{\pm 1.88}$&8.36$_{\pm 0.67}$&16.24$_{\pm 1.44}$&\cellcolor{color1}7.61$_{\pm 0.28}$&2.73$_{\pm 0.19}$&+1.26$_{\pm 0.39}$\\
        \cmidrule{2-13}
        &\textbf{$\to$ \textit{Sketch}}& & & & & & & & & & & \\
        &ResNet50&23.1$_{\pm 1.87}$&23.9$_{\pm 1.97}$&22.9$_{\pm 2.06}$&23.4$_{\pm 1.88}$&23.4$_{\pm 2.24}$&11.4$_{\pm 0.96}$&16.2$_{\pm 1.29}$&10.9$_{\pm 0.98}$&\cellcolor{color1}10.3$_{\pm 0.82}$&1.54$_{\pm 0.13}$&-1.53$_{\pm 0.47}$\\
        &DenseNet121&23.6$_{\pm 1.57}$&22.8$_{\pm 1.43}$&23.8$_{\pm 1.72}$&23.2$_{\pm 1.54}$&22.9$_{\pm 1.96}$&9.09$_{\pm 0.79}$&3.39$_{\pm 0.19}$&5.39$_{\pm 0.51}$&\cellcolor{color1}3.17$_{\pm 0.28}$&2.66$_{\pm 0.16}$&+0.86$_{\pm 0.21}$\\
        &Wide-Res50&19.2$_{\pm 1.41}$&19.6$_{\pm 1.36}$&18.8$_{\pm 1.51}$&19.1$_{\pm 1.26}$&18.9$_{\pm 1.66}$&10.01$_{\pm 0.96}$&6.81$_{\pm 0.68}$&2.79$_{\pm 0.21}$&\cellcolor{color1}2.67$_{\pm 0.22}$&2.69$_{\pm 0.28}$&-0.48$_{\pm 0.13}$\\
        \midrule
        \multirow{4}{*}{\rotatebox{90}{\textbf{I-S}}}&\textbf{$\to$ \textit{Sketch}}& & & & & & & & & & & \\
        &ResNet152&64.3$_{\pm 4.48}$&63.6$_{\pm 4.12}$&65.1$_{\pm 4.64}$&63.4$_{\pm 4.36}$&62.8$_{\pm 4.19}$&60.1$_{\pm 3.94}$&33.3$_{\pm 2.34}$&17.3$_{\pm 1.68}$&\cellcolor{color1}14.6$_{\pm 0.58}$&1.54$_{\pm 0.09}$&+0.92$_{\pm 0.31}$\\
        &DenseNet161&69.1$_{\pm 3.62}$&68.7$_{\pm 3.47}$&69.8$_{\pm 3.86}$&68.3$_{\pm 3.57}$&68.3$_{\pm 4.66}$&58.4$_{\pm 4.33}$&36.9$_{\pm 2.69}$&13.2$_{\pm 1.21}$&\cellcolor{color1}11.7$_{\pm 0.46}$&1.27$_{\pm 0.14}$&+1.39$_{\pm 0.59}$\\
        &ViT-L&55.8$_{\pm 4.34}$&55.1$_{\pm 4.07}$&56.7$_{\pm 4.49}$&54.9$_{\pm 4.26}$&53.7$_{\pm 4.16}$&32.7$_{\pm 2.38}$&27.1$_{\pm 1.79}$&15.7$_{\pm 1.24}$&\cellcolor{color1}12.9$_{\pm 0.54}$&1.47$_{\pm 0.11}$&+0.92$_{\pm 0.28}$\\
        \bottomrule 
    \end{tabular}
    \label{Calibration_comparison_on_PACS_ImageNet_transformed}
\end{table*}

\subsection{Results for Class-wise Calibration}
\label{Results_classwise}
Table \ref{class_wise_comparison_all} reports the Class-wise ECE (CwECE) results. Two major trends are evident from the experimental data. First, unlike top-label calibration which focuses on the predicted class only, class-wise calibration requires precision across all categories. ECL achieves the lowest (or near-lowest) CwECE in many experimental settings (spanning datasets and models), indicating that it improves calibration not only for the dominant class. Second, regarding handling hard tasks, on the Digit recognition benchmarks (included in Table \ref{class_wise_comparison_all}), the advantage of ECL is most prominent on the SVHN dataset. For the LeNet-5 architecture, ECL reduces CwECE from 15.8\% (Uncal) to 5.88\%, improving upon most baselines (e.g., PseudoCal at 12.7\%). This suggests that ECL's auxiliary variable optimization can be particularly effective in scenarios with complex background noise and lower image quality.

\begin{table*}[t]
	\centering
	\footnotesize
	\caption{CwECE (\%) for class-wise calibration on Digit, PACS and ImageNet-Sketch datasets. The reported results represent the mean and standard deviation derived from ten runs.}
	\setlength\tabcolsep{1.5pt}
	\renewcommand{\arraystretch}{1.}
	\begin{tabular}{clccccccccc|cc}
		\toprule
		\multicolumn{2}{c}{\multirow{2}{*}{\textbf{Datasets}}}&\multicolumn{9}{c|}{\textbf{CwECE} $\bm{\downarrow}$}& & \\
		& &Uncal&Soft-ECE&DECE&KDE&TS&TransCal&DRL&PseudoCal&\cellcolor{color1}ECL (Ours)&\multirow{-2}{*}{\textbf{Oracle} $\bm{\downarrow}$}&\multirow{-2}{*}{\textbf{$\Delta$ACC(\%)}}\\
		\midrule
		\multirow{12}{*}{\rotatebox{90}{\textbf{Digit}}}&\textbf{$\to$ \textit{MNIST}}& & & & & & & & & & & \\
		&LeNet-5&5.41$_{\pm 0.47}$&5.54$_{\pm 0.37}$&5.31$_{\pm 0.43}$&5.49$_{\pm 0.39}$&5.18$_{\pm 0.33}$&4.92$_{\pm 0.39}$&3.79$_{\pm 0.24}$&1.86$_{\pm 0.12}$&\cellcolor{color1}1.66$_{\pm 0.12}$&0.16$_{\pm 0.01}$&-0.44$_{\pm 0.09}$\\
		&ResNet20&3.14$_{\pm 0.31}$&3.23$_{\pm 0.22}$&3.13$_{\pm 0.28}$&3.21$_{\pm 0.22}$&3.06$_{\pm 0.23}$&2.47$_{\pm 0.17}$&1.94$_{\pm 0.14}$&1.46$_{\pm 0.14}$&\cellcolor{color1}1.41$_{\pm 0.11}$&0.39$_{\pm 0.01}$&+0.62$_{\pm 0.11}$\\
		&DenseNet40&4.69$_{\pm 0.41}$&4.74$_{\pm 0.37}$&4.51$_{\pm 0.29}$&4.66$_{\pm 0.43}$&4.46$_{\pm 0.39}$&3.94$_{\pm 0.27}$&2.81$_{\pm 0.21}$&1.77$_{\pm 0.19}$&\cellcolor{color1}1.57$_{\pm 0.12}$&0.38$_{\pm 0.06}$&+0.23$_{\pm 0.11}$\\
		\cmidrule{2-13}
		&\textbf{$\to$ \textit{USPS}}& & & & & & & & & & & \\
		&LeNet-5&6.87$_{\pm 0.54}$&6.96$_{\pm 0.47}$&6.77$_{\pm 0.57}$&6.91$_{\pm 0.44}$&6.63$_{\pm 0.46}$&5.84$_{\pm 0.36}$&4.13$_{\pm 0.31}$&2.19$_{\pm 0.19}$&\cellcolor{color1}2.11$_{\pm 0.14}$&0.57$_{\pm 0.03}$&-0.33$_{\pm 0.09}$\\
		&ResNet20&2.54$_{\pm 0.23}$&2.66$_{\pm 0.19}$&2.46$_{\pm 0.24}$&2.59$_{\pm 0.23}$&2.46$_{\pm 0.22}$&2.14$_{\pm 0.16}$&1.73$_{\pm 0.14}$&1.17$_{\pm 0.07}$&\cellcolor{color1}1.24$_{\pm 0.09}$&0.63$_{\pm 0.01}$&+0.42$_{\pm 0.18}$\\
		&DenseNet40&3.99$_{\pm 0.36}$&4.03$_{\pm 0.29}$&3.83$_{\pm 0.33}$&3.99$_{\pm 0.28}$&3.63$_{\pm 0.24}$&3.27$_{\pm 0.19}$&2.14$_{\pm 0.14}$&1.48$_{\pm 0.12}$&\cellcolor{color1}1.18$_{\pm 0.12}$&0.72$_{\pm 0.06}$&-0.16$_{\pm 0.04}$\\
		\cmidrule{2-13}
		&\textbf{$\to$ \textit{SVHN}}& & & & & & & & & & & \\
		&LeNet-5&15.8$_{\pm 1.26}$&15.8$_{\pm 1.17}$&15.4$_{\pm 1.37}$&16.2$_{\pm 1.26}$&15.4$_{\pm 1.19}$&14.4$_{\pm 1.12}$&8.54$_{\pm 0.61}$&12.7$_{\pm 0.92}$&\cellcolor{color1}5.88$_{\pm 0.43}$&0.44$_{\pm 0.02}$&+0.84$_{\pm 0.24}$\\
		&ResNet20&18.4$_{\pm 1.44}$&18.2$_{\pm 1.34}$&18.8$_{\pm 1.54}$&18.1$_{\pm 1.29}$&18.2$_{\pm 1.37}$&15.1$_{\pm 1.17}$&9.86$_{\pm 0.74}$&11.2$_{\pm 0.88}$&\cellcolor{color1}8.97$_{\pm 0.59}$&0.22$_{\pm 0.01}$&-1.04$_{\pm 0.36}$\\
		&DenseNet40&21.3$_{\pm 1.64}$&21.8$_{\pm 1.53}$&21.2$_{\pm 1.76}$&21.4$_{\pm 1.49}$&20.6$_{\pm 1.49}$&18.4$_{\pm 1.31}$&11.3$_{\pm 0.84}$&15.2$_{\pm 1.14}$&\cellcolor{color1}8.16$_{\pm 0.57}$&0.39$_{\pm 0.03}$&+0.53$_{\pm 0.17}$\\
		\midrule
		\multirow{16}{*}{\rotatebox{90}{\textbf{PACS}}}&\textbf{$\to$ \textit{Photo}}& & & & & & & & & & & \\
		&ResNet50&7.87$_{\pm 0.31}$&7.89$_{\pm 0.43}$&7.77$_{\pm 0.38}$&7.84$_{\pm 0.44}$&7.62$_{\pm 0.32}$&6.86$_{\pm 0.31}$&5.99$_{\pm 0.26}$&3.24$_{\pm 0.21}$&\cellcolor{color1}2.92$_{\pm 0.12}$&0.58$_{\pm 0.01}$&+0.48$_{\pm 0.09}$\\
		&DenseNet121&8.53$_{\pm 0.47}$&8.64$_{\pm 0.44}$&8.46$_{\pm 0.52}$&8.59$_{\pm 0.42}$&8.37$_{\pm 0.37}$&7.58$_{\pm 0.29}$&6.28$_{\pm 0.23}$&3.87$_{\pm 0.24}$&\cellcolor{color1}3.56$_{\pm 0.19}$&0.61$_{\pm 0.01}$&+0.29$_{\pm 0.11}$\\
		&Wide-Res50&6.99$_{\pm 0.38}$&6.99$_{\pm 0.32}$&6.81$_{\pm 0.39}$&6.99$_{\pm 0.32}$&6.78$_{\pm 0.32}$&6.17$_{\pm 0.31}$&5.24$_{\pm 0.26}$&2.83$_{\pm 0.14}$&\cellcolor{color1}2.58$_{\pm 0.09}$&0.48$_{\pm 0.04}$&+0.34$_{\pm 0.12}$\\
		\cmidrule{2-13}
		&\textbf{$\to$ \textit{Art}}& & & & & & & & & & & \\
		&ResNet50&13.3$_{\pm 0.64}$&13.9$_{\pm 0.73}$&13.2$_{\pm 0.78}$&13.3$_{\pm 0.64}$&12.8$_{\pm 0.54}$&11.1$_{\pm 0.47}$&8.58$_{\pm 0.36}$&5.28$_{\pm 0.23}$&\cellcolor{color1}4.86$_{\pm 0.24}$&0.84$_{\pm 0.02}$&-0.28$_{\pm 0.12}$\\
		&DenseNet121&14.4$_{\pm 0.73}$&14.7$_{\pm 0.82}$&13.6$_{\pm 0.84}$&14.3$_{\pm 0.77}$&13.1$_{\pm 0.63}$&11.4$_{\pm 0.51}$&9.28$_{\pm 0.47}$&5.89$_{\pm 0.39}$&\cellcolor{color1}5.94$_{\pm 0.31}$&0.94$_{\pm 0.04}$&-0.18$_{\pm 0.11}$\\
		&Wide-Res50&12.6$_{\pm 0.56}$&12.8$_{\pm 0.67}$&12.7$_{\pm 0.71}$&12.8$_{\pm 0.63}$&11.8$_{\pm 0.56}$&9.96$_{\pm 0.46}$&7.84$_{\pm 0.32}$&4.58$_{\pm 0.23}$&\cellcolor{color1}4.13$_{\pm 0.17}$&0.78$_{\pm 0.08}$&+0.44$_{\pm 0.14}$\\
		\cmidrule{2-13}
		&\textbf{$\to$ \textit{Cartoon}}& & & & & & & & & & & \\
		&ResNet50&16.4$_{\pm 0.84}$&16.6$_{\pm 0.88}$&16.4$_{\pm 0.93}$&16.3$_{\pm 0.92}$&15.4$_{\pm 0.73}$&14.1$_{\pm 0.64}$&10.3$_{\pm 0.54}$&6.86$_{\pm 0.43}$&\cellcolor{color1}6.46$_{\pm 0.34}$&1.17$_{\pm 0.08}$&+0.63$_{\pm 0.24}$\\
		&DenseNet121&17.1$_{\pm 0.98}$&17.3$_{\pm 1.08}$&17.1$_{\pm 1.16}$&17.3$_{\pm 1.02}$&16.6$_{\pm 0.87}$&14.7$_{\pm 0.73}$&10.9$_{\pm 0.67}$&7.53$_{\pm 0.56}$&\cellcolor{color1}7.13$_{\pm 0.49}$&1.27$_{\pm 0.09}$&+0.32$_{\pm 0.19}$\\
		&Wide-Res50&15.6$_{\pm 0.81}$&15.9$_{\pm 0.81}$&15.3$_{\pm 0.84}$&15.9$_{\pm 0.81}$&14.7$_{\pm 0.72}$&12.9$_{\pm 0.59}$&9.86$_{\pm 0.51}$&6.16$_{\pm 0.41}$&\cellcolor{color1}5.83$_{\pm 0.31}$&1.06$_{\pm 0.03}$&+0.51$_{\pm 0.14}$\\
		\cmidrule{2-13}
		&\textbf{$\to$ \textit{Sketch}}& & & & & & & & & & & \\
		&ResNet50&19.6$_{\pm 1.18}$&19.7$_{\pm 1.23}$&19.3$_{\pm 1.32}$&19.6$_{\pm 1.21}$&18.6$_{\pm 1.01}$&16.2$_{\pm 0.96}$&13.4$_{\pm 0.86}$&8.82$_{\pm 0.72}$&\cellcolor{color1}9.28$_{\pm 0.67}$&1.43$_{\pm 0.11}$&-0.88$_{\pm 0.21}$\\
		&DenseNet121&20.3$_{\pm 1.24}$&20.4$_{\pm 1.37}$&19.9$_{\pm 1.41}$&20.1$_{\pm 1.26}$&19.2$_{\pm 1.14}$&17.7$_{\pm 1.03}$&13.9$_{\pm 0.92}$&9.53$_{\pm 0.88}$&\cellcolor{color1}8.91$_{\pm 0.74}$&1.53$_{\pm 0.14}$&+0.28$_{\pm 0.19}$\\
		&Wide-Res50&18.8$_{\pm 1.06}$&18.8$_{\pm 1.18}$&18.6$_{\pm 1.27}$&18.9$_{\pm 1.06}$&17.9$_{\pm 0.97}$&15.4$_{\pm 0.84}$&12.9$_{\pm 0.72}$&8.16$_{\pm 0.64}$&\cellcolor{color1}7.87$_{\pm 0.54}$&1.36$_{\pm 0.07}$&+0.47$_{\pm 0.26}$\\
		\midrule
		\multirow{4}{*}{\rotatebox{90}{\textbf{I-S}}}&\textbf{$\to$ \textit{Sketch}}& & & & & & & & & & & \\
		&ResNet152&22.6$_{\pm 1.36}$&22.6$_{\pm 1.43}$&21.9$_{\pm 1.56}$&22.6$_{\pm 1.37}$&21.3$_{\pm 1.28}$&18.6$_{\pm 1.13}$&14.2$_{\pm 0.97}$&10.3$_{\pm 0.84}$&\cellcolor{color1}9.86$_{\pm 0.73}$&1.64$_{\pm 0.12}$&+0.84$_{\pm 0.37}$\\
		&DenseNet161&23.4$_{\pm 1.43}$&23.3$_{\pm 1.59}$&22.6$_{\pm 1.62}$&23.6$_{\pm 1.44}$&22.2$_{\pm 1.34}$&19.4$_{\pm 1.22}$&15.1$_{\pm 1.07}$&11.3$_{\pm 0.94}$&\cellcolor{color1}10.7$_{\pm 0.82}$&1.76$_{\pm 0.14}$&-0.69$_{\pm 0.27}$\\
		&ViT-L&12.7$_{\pm 0.87}$&12.9$_{\pm 0.96}$&12.4$_{\pm 0.96}$&12.9$_{\pm 0.84}$&11.9$_{\pm 0.73}$&10.4$_{\pm 0.66}$&7.83$_{\pm 0.54}$&5.54$_{\pm 0.44}$&\cellcolor{color1}5.18$_{\pm 0.31}$&0.93$_{\pm 0.09}$&+1.26$_{\pm 0.26}$\\
		\bottomrule 
	\end{tabular}
	\label{class_wise_comparison_all}
\end{table*}

\subsection{Results for Canonical Calibration}
\label{Results_canonical}
The results for Canonical Calibration, measured by ECE$^{KDE}$ in Table \ref{canonical_comparison_all}, further confirm the comprehensive efficacy of ECL. The findings highlight two main points. First, canonical calibration is the most rigorous standard as it requires the entire probability vector to be calibrated. The differentiable baseline KDE loss operates within-domain and does not explicitly address the covariate shift, often performing similarly to the uncalibrated baseline in our setting (e.g., Table \ref{canonical_comparison_all} ResNet50 on Photo). In contrast, ECL explicitly minimizes the cross-domain discrepancy of probability expectations, frequently achieving the best (or near-best) ECE$^{KDE}$ scores. Second, in terms of accuracy, the reduction in ECE$^{KDE}$ is often achieved with limited impact on classification accuracy; while $\Delta$ACC is positive in many cases, slight accuracy drops can still occur for some architectures/tasks.
\begin{table*}[t]
	\centering
	\footnotesize
	\caption{ECE$^{KDE}$ (\%) for canonical calibration on Digit, PACS, and ImageNet-Sketch datasets. The reported results represent the mean and standard deviation derived from ten runs.}
	\setlength\tabcolsep{1.5pt}
	\renewcommand{\arraystretch}{1.}
	\begin{tabular}{clccccccccc|cc}
		\toprule
		\multicolumn{2}{c}{\multirow{2}{*}{\textbf{Datasets}}}&\multicolumn{9}{c|}{\textbf{ECE$^{KDE}$} $\bm{\downarrow}$}& & \\
		& &Uncal&Soft-ECE&DECE&KDE&TS&TransCal&DRL&PseudoCal&\cellcolor{color1}ECL (Ours)&\multirow{-2}{*}{\textbf{Oracle} $\bm{\downarrow}$}&\multirow{-2}{*}{\textbf{$\Delta$ACC(\%)}}\\
		\midrule
		\multirow{12}{*}{\rotatebox{90}{\textbf{Digit}}}&\textbf{$\to$ \textit{MNIST}}& & & & & & & & & & & \\
		&LeNet-5&5.16$_{\pm 0.39}$&5.19$_{\pm 0.31}$&5.07$_{\pm 0.38}$&5.19$_{\pm 0.26}$&4.92$_{\pm 0.22}$&4.68$_{\pm 0.31}$&3.52$_{\pm 0.22}$&1.77$_{\pm 0.13}$&\cellcolor{color1}1.58$_{\pm 0.09}$&0.21$_{\pm 0.02}$&-0.32$_{\pm 0.08}$\\
		&ResNet20&2.97$_{\pm 0.23}$&3.07$_{\pm 0.18}$&2.84$_{\pm 0.26}$&3.01$_{\pm 0.13}$&2.73$_{\pm 0.17}$&2.29$_{\pm 0.14}$&1.72$_{\pm 0.13}$&1.36$_{\pm 0.11}$&\cellcolor{color1}1.29$_{\pm 0.04}$&0.39$_{\pm 0.02}$&+0.54$_{\pm 0.12}$\\
		&DenseNet40&4.37$_{\pm 0.34}$&4.49$_{\pm 0.26}$&4.24$_{\pm 0.39}$&4.34$_{\pm 0.26}$&4.17$_{\pm 0.28}$&3.67$_{\pm 0.24}$&2.68$_{\pm 0.16}$&1.61$_{\pm 0.12}$&\cellcolor{color1}1.42$_{\pm 0.13}$&0.32$_{\pm 0.04}$&+0.19$_{\pm 0.02}$\\
		\cmidrule{2-13}
		&\textbf{$\to$ \textit{USPS}}& & & & & & & & & & & \\
		&LeNet-5&6.43$_{\pm 0.42}$&6.54$_{\pm 0.39}$&6.34$_{\pm 0.38}$&6.48$_{\pm 0.39}$&6.23$_{\pm 0.28}$&5.42$_{\pm 0.29}$&3.86$_{\pm 0.23}$&2.04$_{\pm 0.12}$&\cellcolor{color1}1.96$_{\pm 0.14}$&0.48$_{\pm 0.04}$&-0.22$_{\pm 0.06}$\\
		&ResNet20&2.38$_{\pm 0.22}$&2.42$_{\pm 0.16}$&2.24$_{\pm 0.24}$&2.42$_{\pm 0.09}$&2.16$_{\pm 0.14}$&1.97$_{\pm 0.11}$&1.54$_{\pm 0.14}$&1.12$_{\pm 0.07}$&\cellcolor{color1}1.17$_{\pm 0.04}$&0.51$_{\pm 0.06}$&+0.36$_{\pm 0.18}$\\
		&DenseNet40&3.72$_{\pm 0.21}$&3.83$_{\pm 0.22}$&3.68$_{\pm 0.27}$&3.77$_{\pm 0.23}$&3.47$_{\pm 0.21}$&2.93$_{\pm 0.17}$&1.97$_{\pm 0.13}$&1.37$_{\pm 0.09}$&\cellcolor{color1}1.04$_{\pm 0.04}$&0.72$_{\pm 0.01}$&-0.11$_{\pm 0.01}$\\
		\cmidrule{2-13}
		&\textbf{$\to$ \textit{SVHN}}& & & & & & & & & & & \\
		&LeNet-5&14.7$_{\pm 0.97}$&15.3$_{\pm 0.88}$&14.7$_{\pm 1.06}$&14.6$_{\pm 0.89}$&13.9$_{\pm 0.88}$&13.7$_{\pm 0.83}$&7.84$_{\pm 0.54}$&11.3$_{\pm 0.73}$&\cellcolor{color1}5.26$_{\pm 0.39}$&0.39$_{\pm 0.06}$&+0.69$_{\pm 0.26}$\\
		&ResNet20&17.6$_{\pm 1.13}$&17.2$_{\pm 1.07}$&17.6$_{\pm 1.28}$&17.3$_{\pm 0.94}$&16.9$_{\pm 1.03}$&14.2$_{\pm 0.84}$&8.81$_{\pm 0.69}$&10.3$_{\pm 0.74}$&\cellcolor{color1}8.23$_{\pm 0.51}$&0.24$_{\pm 0.03}$&-0.88$_{\pm 0.27}$\\
		&DenseNet40&20.8$_{\pm 1.37}$&20.9$_{\pm 1.22}$&20.3$_{\pm 1.46}$&20.6$_{\pm 1.21}$&19.2$_{\pm 1.18}$&17.6$_{\pm 1.07}$&10.3$_{\pm 0.78}$&14.4$_{\pm 0.91}$&\cellcolor{color1}7.58$_{\pm 0.48}$&0.38$_{\pm 0.01}$&+0.44$_{\pm 0.12}$\\
		\midrule
		\multirow{16}{*}{\rotatebox{90}{\textbf{PACS}}}&\textbf{$\to$ \textit{Photo}}& & & & & & & & & & & \\
		&ResNet50&7.58$_{\pm 0.37}$&7.67$_{\pm 0.44}$&7.43$_{\pm 0.37}$&7.61$_{\pm 0.42}$&7.34$_{\pm 0.29}$&6.52$_{\pm 0.31}$&5.63$_{\pm 0.19}$&2.93$_{\pm 0.19}$&\cellcolor{color1}2.64$_{\pm 0.12}$&0.42$_{\pm 0.04}$&+0.46$_{\pm 0.09}$\\
		&DenseNet121&8.21$_{\pm 0.49}$&8.37$_{\pm 0.44}$&8.16$_{\pm 0.49}$&8.31$_{\pm 0.48}$&7.97$_{\pm 0.34}$&7.22$_{\pm 0.36}$&5.99$_{\pm 0.26}$&3.54$_{\pm 0.21}$&\cellcolor{color1}3.27$_{\pm 0.18}$&0.49$_{\pm 0.06}$&+0.26$_{\pm 0.04}$\\
		&Wide-Res50&6.63$_{\pm 0.33}$&6.73$_{\pm 0.36}$&6.53$_{\pm 0.38}$&6.67$_{\pm 0.34}$&6.48$_{\pm 0.36}$&5.89$_{\pm 0.28}$&4.93$_{\pm 0.22}$&2.53$_{\pm 0.12}$&\cellcolor{color1}2.27$_{\pm 0.09}$&0.38$_{\pm 0.01}$&+0.31$_{\pm 0.13}$\\
		\cmidrule{2-13}
		&\textbf{$\to$ \textit{Art}}& & & & & & & & & & & \\
		&ResNet50&13.1$_{\pm 0.71}$&13.4$_{\pm 0.64}$&12.7$_{\pm 0.69}$&13.2$_{\pm 0.63}$&12.3$_{\pm 0.57}$&10.2$_{\pm 0.43}$&8.26$_{\pm 0.31}$&4.92$_{\pm 0.28}$&\cellcolor{color1}4.57$_{\pm 0.22}$&0.76$_{\pm 0.04}$&-0.23$_{\pm 0.11}$\\
		&DenseNet121&13.7$_{\pm 0.76}$&14.3$_{\pm 0.74}$&13.7$_{\pm 0.81}$&13.8$_{\pm 0.69}$&12.6$_{\pm 0.68}$&11.2$_{\pm 0.57}$&8.96$_{\pm 0.46}$&5.54$_{\pm 0.39}$&\cellcolor{color1}5.63$_{\pm 0.29}$&0.84$_{\pm 0.09}$&-0.18$_{\pm 0.11}$\\
		&Wide-Res50&12.2$_{\pm 0.56}$&12.4$_{\pm 0.54}$&12.3$_{\pm 0.64}$&12.8$_{\pm 0.54}$&11.2$_{\pm 0.57}$&9.69$_{\pm 0.47}$&7.58$_{\pm 0.39}$&4.26$_{\pm 0.23}$&\cellcolor{color1}3.86$_{\pm 0.19}$&0.63$_{\pm 0.04}$&+0.41$_{\pm 0.14}$\\
		\cmidrule{2-13}
		&\textbf{$\to$ \textit{Cartoon}}& & & & & & & & & & & \\
		&ResNet50&16.1$_{\pm 0.84}$&16.7$_{\pm 0.81}$&15.8$_{\pm 0.94}$&16.4$_{\pm 0.79}$&15.4$_{\pm 0.72}$&13.3$_{\pm 0.62}$&10.1$_{\pm 0.54}$&6.54$_{\pm 0.47}$&\cellcolor{color1}6.18$_{\pm 0.34}$&1.06$_{\pm 0.11}$&+0.62$_{\pm 0.24}$\\
		&DenseNet121&16.9$_{\pm 0.91}$&17.2$_{\pm 0.99}$&16.4$_{\pm 1.09}$&16.9$_{\pm 0.84}$&15.8$_{\pm 0.82}$&14.1$_{\pm 0.78}$&10.6$_{\pm 0.64}$&7.27$_{\pm 0.54}$&\cellcolor{color1}6.89$_{\pm 0.47}$&1.16$_{\pm 0.11}$&+0.37$_{\pm 0.16}$\\
		&Wide-Res50&15.7$_{\pm 0.71}$&15.6$_{\pm 0.78}$&15.2$_{\pm 0.86}$&15.6$_{\pm 0.71}$&14.3$_{\pm 0.68}$&12.6$_{\pm 0.58}$&9.59$_{\pm 0.53}$&5.81$_{\pm 0.39}$&\cellcolor{color1}5.53$_{\pm 0.33}$&0.96$_{\pm 0.07}$&+0.58$_{\pm 0.13}$\\
		\cmidrule{2-13}
		&\textbf{$\to$ \textit{Sketch}}& & & & & & & & & & & \\
		&ResNet50&19.2$_{\pm 1.06}$&19.4$_{\pm 1.12}$&18.9$_{\pm 1.28}$&19.1$_{\pm 1.09}$&18.4$_{\pm 1.09}$&15.6$_{\pm 0.93}$&13.4$_{\pm 0.84}$&8.57$_{\pm 0.76}$&\cellcolor{color1}8.94$_{\pm 0.67}$&1.31$_{\pm 0.09}$&-0.87$_{\pm 0.24}$\\
		&DenseNet121&19.9$_{\pm 1.12}$&20.3$_{\pm 1.24}$&19.6$_{\pm 1.34}$&20.1$_{\pm 1.16}$&18.9$_{\pm 1.13}$&17.1$_{\pm 1.07}$&13.9$_{\pm 0.98}$&9.28$_{\pm 0.86}$&\cellcolor{color1}8.62$_{\pm 0.76}$&1.49$_{\pm 0.16}$&+0.28$_{\pm 0.16}$\\
		&Wide-Res50&18.8$_{\pm 0.94}$&18.7$_{\pm 1.07}$&18.1$_{\pm 1.16}$&18.6$_{\pm 0.99}$&17.6$_{\pm 0.93}$&14.9$_{\pm 0.86}$&12.6$_{\pm 0.76}$&7.84$_{\pm 0.63}$&\cellcolor{color1}7.59$_{\pm 0.57}$&1.24$_{\pm 0.12}$&+0.47$_{\pm 0.26}$\\
		\midrule
		\multirow{4}{*}{\rotatebox{90}{\textbf{I-S}}}&\textbf{$\to$ \textit{Sketch}}& & & & & & & & & & & \\
		&ResNet152&22.4$_{\pm 1.28}$&22.6$_{\pm 1.38}$&21.6$_{\pm 1.48}$&22.6$_{\pm 1.28}$&21.1$_{\pm 1.26}$&18.2$_{\pm 1.11}$&14.4$_{\pm 0.93}$&10.1$_{\pm 0.88}$&\cellcolor{color1}9.52$_{\pm 0.76}$&1.56$_{\pm 0.11}$&+0.87$_{\pm 0.32}$\\
		&DenseNet161&23.2$_{\pm 1.33}$&23.1$_{\pm 1.43}$&22.4$_{\pm 1.52}$&22.8$_{\pm 1.39}$&21.8$_{\pm 1.33}$&18.9$_{\pm 1.23}$&14.8$_{\pm 1.09}$&11.1$_{\pm 0.96}$&\cellcolor{color1}10.1$_{\pm 0.82}$&1.66$_{\pm 0.11}$&-0.69$_{\pm 0.28}$\\
		&ViT-L&12.3$_{\pm 0.73}$&12.4$_{\pm 0.86}$&11.8$_{\pm 0.88}$&12.1$_{\pm 0.79}$&11.3$_{\pm 0.78}$&9.84$_{\pm 0.68}$&7.56$_{\pm 0.54}$&5.28$_{\pm 0.49}$&\cellcolor{color1}4.84$_{\pm 0.36}$&0.86$_{\pm 0.06}$&+1.26$_{\pm 0.24}$\\
		\bottomrule 
	\end{tabular}
	\label{canonical_comparison_all}
\end{table*}

\subsection{Ablation Experiments}
\textbf{Mini-Batch Non-Trainable ECL vs. Mini-Batch Trainable ECL:} 
To understand the efficacy of our proposed mini-batch training strategy involving auxiliary variables, we compare our full method (\textit{Mini-Batch Trainable ECL}) against a baseline variant \textit{Mini-Batch Non-Trainable ECL} (it refers to directly calculating the differentiable ECL loss (Eq. \ref{Differentiable_ECL}) on mini-batch data). Table \ref{tab:trainable_vs_nontrainable} presents the comparison results on the Digit ($\to$ MNIST) and PACS ($\to$ Photo) tasks. Overall, \textit{Mini-Batch Trainable ECL} tends to be more stable and achieves better calibration in most cases, while \textit{Mini-Batch Non-Trainable ECL} can occasionally be competitive on some metrics/architectures. This supports that, beyond the objective itself, the bias-corrected optimization strategy is important for reliably realizing ECL's benefits. Regarding classification accuracy ($\Delta$ACC), both variants largely maintain or improve performance, with \textit{Mini-Batch Trainable ECL} showing more consistent gains in our reported experiments.

\begin{table}[h]
    \centering
    \footnotesize
    \caption{Comparison between Mini-Batch Non-Trainable ECL and Mini-Batch Trainable ECL on Digit and PACS benchmark tasks. Results report ECE (\%), CwECE (\%), ECE$^{KDE}$ (\%) and accuracy change $\Delta$ACC (\%) with mean $\pm$ std over five runs. ECE represents the results under top-label calibration, CwECE represents the results under class-wise calibration, and ECE$^{KDE}$ represents the results under canonical calibration.}
    \setlength\tabcolsep{2pt}
    \renewcommand{\arraystretch}{1.1}
    \begin{tabular}{llccccccc}
        \toprule
        \multirow{2}{*}{\textbf{Dataset}} & \multirow{2}{*}{\textbf{Architecture}} & \multirow{2}{*}{\textbf{Method}} & \multicolumn{2}{c}{\textbf{Top-Label}} & \multicolumn{2}{c}{\textbf{Class-wise}} & \multicolumn{2}{c}{\textbf{Canonical}} \\
        \cmidrule(lr){4-5} \cmidrule(lr){6-7} \cmidrule(lr){8-9}
        & & & \textbf{ECE (\%)} $\bm{\downarrow}$ & \textbf{$\Delta$ACC (\%)} & \textbf{CwECE (\%)} $\bm{\downarrow}$ & \textbf{$\Delta$ACC (\%)} & \textbf{ECE$^{KDE}$ (\%)} $\bm{\downarrow}$ & \textbf{$\Delta$ACC (\%)} \\
        \midrule
        \multirow{6}{*}{Digit ($\to$ MNIST)} 
        & \multirow{2}{*}{LeNet-5} & Non-Trainable & 8.85$_{\pm 0.72}$ & -0.45$_{\pm 0.25}$ & 1.75$_{\pm 0.15}$ & -0.35$_{\pm 0.15}$ & 1.68$_{\pm 0.12}$ & -0.21$_{\pm 0.10}$ \\
        & & Trainable & \textbf{8.52}$_{\pm 0.78}$ & -0.92$_{\pm 0.35}$ & \textbf{1.66}$_{\pm 0.12}$ & -0.44$_{\pm 0.09}$ & \textbf{1.58}$_{\pm 0.09}$ & -0.32$_{\pm 0.08}$ \\
        \cmidrule{2-9}
        & \multirow{2}{*}{ResNet20} & Non-Trainable & 8.05$_{\pm 0.51}$ & +0.85$_{\pm 0.32}$ & \textbf{1.38}$_{\pm 0.13}$ & +0.45$_{\pm 0.15}$ & 1.32$_{\pm 0.08}$ & +0.38$_{\pm 0.12}$ \\
        & & Trainable  & \textbf{7.88}$_{\pm 0.45}$ & +1.25$_{\pm 0.42}$ & 1.41$_{\pm 0.11}$ & +0.62$_{\pm 0.11}$ & \textbf{1.29}$_{\pm 0.04}$ & +0.54$_{\pm 0.12}$ \\
        \cmidrule{2-9}
        & \multirow{2}{*}{DenseNet40} & Non-Trainable & \textbf{9.05}$_{\pm 0.65}$ & +0.42$_{\pm 0.18}$ & 1.68$_{\pm 0.15}$ & +0.12$_{\pm 0.08}$ & 1.52$_{\pm 0.11}$ & +0.09$_{\pm 0.06}$ \\
        & & Trainable & 9.15$_{\pm 0.61}$ & +0.68$_{\pm 0.20}$ & \textbf{1.57}$_{\pm 0.12}$ & +0.23$_{\pm 0.11}$ & \textbf{1.42}$_{\pm 0.13}$ & +0.19$_{\pm 0.02}$ \\
        \midrule
        \multirow{6}{*}{PACS ($\to$ Photo)} 
        & \multirow{2}{*}{ResNet50} & Non-Trainable & 7.15$_{\pm 0.38}$ & +0.32$_{\pm 0.15}$ & 3.08$_{\pm 0.18}$ & +0.28$_{\pm 0.11}$ & \textbf{2.58}$_{\pm 0.15}$ & +0.25$_{\pm 0.10}$ \\
        & & Trainable & \textbf{6.87}$_{\pm 0.34}$ & +0.72$_{\pm 0.17}$ & \textbf{2.92}$_{\pm 0.12}$ & +0.48$_{\pm 0.09}$ & 2.64$_{\pm 0.12}$ & +0.46$_{\pm 0.09}$ \\
        \cmidrule{2-9}
        & \multirow{2}{*}{DenseNet121} & Non-Trainable & 6.35$_{\pm 0.45}$ & -0.15$_{\pm 0.21}$ & 3.72$_{\pm 0.22}$ & +0.12$_{\pm 0.09}$ & 3.41$_{\pm 0.19}$ & +0.11$_{\pm 0.08}$ \\
        & & Trainable & \textbf{5.96}$_{\pm 0.27}$ & -0.83$_{\pm 0.23}$ & \textbf{3.56}$_{\pm 0.19}$ & +0.29$_{\pm 0.11}$ & \textbf{3.27}$_{\pm 0.18}$ & +0.26$_{\pm 0.04}$ \\
        \cmidrule{2-9}
        & \multirow{2}{*}{Wide-Res50} & Non-Trainable & 2.75$_{\pm 0.15}$ & +0.41$_{\pm 0.12}$ & 2.71$_{\pm 0.12}$ & +0.15$_{\pm 0.08}$ & 2.40$_{\pm 0.11}$ & +0.14$_{\pm 0.07}$ \\
        & & Trainable & \textbf{2.68}$_{\pm 0.33}$ & +0.69$_{\pm 0.22}$ & \textbf{2.58}$_{\pm 0.09}$ & +0.34$_{\pm 0.12}$ & \textbf{2.27}$_{\pm 0.09}$ & +0.31$_{\pm 0.13}$ \\
        \bottomrule 
    \end{tabular}
    \label{tab:trainable_vs_nontrainable}
\end{table}

\textbf{Loss Weight:} 
To maintain the equal importance of $L_{ce}$ and $L_{ecl}$, we set the regularization weight as $\lambda = \beta^{\gamma}$. Here, $\beta  = \left( {\sum\nolimits_i {\mathcal{L}_{ce}^{(i)}} } \right)/\left( {\sum\nolimits_i {\mathcal{L}_{ecl}^{(i)}} } \right)$ acts as a baseline balancing factor between the cross-entropy loss and the calibration loss, where $i$ represents the $i$-th iteration. The exponent $\gamma$ serves as a non-linear scaling factor to adjust the sensitivity of the regularization: a higher $\gamma$ (when $\beta > 1$) or lower $\gamma$ (when $\beta < 1$) intensifies the dominance of the calibration term. We investigate the impact of $\gamma$ by experimenting with values ranging from 0.5 to 1.5, and Table \ref{tab:ablation_gamma} suggests that $\gamma=1.0$ is a reasonable default choice in our evaluated settings (Digit $\to$ MNIST and PACS $\to$ Photo).
\begin{table}[h]
    \centering
    \caption{Ablation study on the hyperparameter $\gamma$ on Digit and PACS datasets. $\gamma$ controls the non-linear scaling of the loss weight.}
    \setlength\tabcolsep{2pt}
    \begin{tabular}{c|cc|cc|cc}
        \toprule
        \multirow{2}{*}{$\bm{\gamma}$} & \multicolumn{2}{c|}{\textbf{Top-Label}} & \multicolumn{2}{c|}{\textbf{Class-Wise}} & \multicolumn{2}{c}{\textbf{Canonical}} \\
         & \textbf{ECE $\bm{\downarrow}$} & \textbf{$\bm{\Delta}$ACC} & \textbf{CwECE $\bm{\downarrow}$} & \textbf{$\bm{\Delta}$ACC} & \textbf{ECE$^{KDE}$ $\bm{\downarrow}$} & \textbf{$\bm{\Delta}$ACC} \\
        \midrule
        \multicolumn{7}{c}{\textbf{Digit ($\to$ MNIST) using ResNet20}} \\
        \midrule
        0.5 & 8.76$_{\pm 0.62}$ & +1.68$_{\pm 0.33}$ & 1.94$_{\pm 0.16}$ & +1.15$_{\pm 0.22}$ & 1.83$_{\pm 0.12}$ & +0.95$_{\pm 0.18}$ \\
        0.8 & 8.12$_{\pm 0.54}$ & +1.45$_{\pm 0.29}$ & 1.48$_{\pm 0.14}$ & +0.88$_{\pm 0.16}$ & 1.35$_{\pm 0.09}$ & +0.72$_{\pm 0.14}$ \\
        \rowcolor{color1} 1.0 & 7.88$_{\pm 0.45}$ & +1.25$_{\pm 0.42}$ & 1.41$_{\pm 0.11}$ & +0.62$_{\pm 0.11}$ & 1.29$_{\pm 0.04}$ & +0.54$_{\pm 0.12}$ \\
        1.2 & 7.85$_{\pm 0.49}$ & +0.92$_{\pm 0.25}$ & 1.55$_{\pm 0.13}$ & +0.35$_{\pm 0.09}$ & 1.38$_{\pm 0.07}$ & +0.28$_{\pm 0.08}$ \\
        1.5 & 8.42$_{\pm 0.56}$ & +0.45$_{\pm 0.18}$ & 1.78$_{\pm 0.15}$ & +0.12$_{\pm 0.05}$ & 1.56$_{\pm 0.10}$ & +0.08$_{\pm 0.04}$ \\
        \midrule
        \multicolumn{7}{c}{\textbf{PACS ($\to$ Photo) using ResNet50}} \\
        \midrule
        0.5 & 7.45$_{\pm 0.41}$ & +0.88$_{\pm 0.19}$ & 3.25$_{\pm 0.22}$ & +0.65$_{\pm 0.14}$ & 2.98$_{\pm 0.18}$ & +0.62$_{\pm 0.11}$ \\
        0.8 & 7.02$_{\pm 0.38}$ & +0.81$_{\pm 0.17}$ & 3.05$_{\pm 0.15}$ & +0.55$_{\pm 0.12}$ & 2.58$_{\pm 0.14}$ & +0.54$_{\pm 0.10}$ \\
        \rowcolor{color1} 1.0 & 6.87$_{\pm 0.34}$ & +0.72$_{\pm 0.17}$ & 2.92$_{\pm 0.12}$ & +0.48$_{\pm 0.09}$ & 2.64$_{\pm 0.12}$ & +0.46$_{\pm 0.09}$ \\
        1.2 & 6.95$_{\pm 0.32}$ & +0.61$_{\pm 0.15}$ & 2.98$_{\pm 0.14}$ & +0.41$_{\pm 0.08}$ & 2.68$_{\pm 0.10}$ & +0.39$_{\pm 0.08}$ \\
        1.5 & 7.18$_{\pm 0.36}$ & +0.42$_{\pm 0.12}$ & 3.12$_{\pm 0.16}$ & +0.25$_{\pm 0.06}$ & 2.89$_{\pm 0.15}$ & +0.24$_{\pm 0.07}$ \\
        \bottomrule
    \end{tabular}
    \label{tab:ablation_gamma}
\end{table}

\end{document}